\documentclass[11pt]{article}

\usepackage[preprint]{acl}
\usepackage{booktabs}
\usepackage{times}
\usepackage{latexsym}
\usepackage{amsmath}
\usepackage{multirow}
\usepackage{amssymb} 
\usepackage[T1]{fontenc}

\usepackage[utf8]{inputenc}

\usepackage{microtype}
\usepackage{pifont}
\usepackage{inconsolata}

\usepackage{amsthm}

\newtheorem{proposition}{Proposition}

\usepackage{listings}
\usepackage[framemethod=TikZ]{mdframed} 
\newmdenv[
  roundcorner=6pt,
  linecolor=black,
  linewidth=0.8pt,
  backgroundcolor=white,
  innerleftmargin=6pt, innerrightmargin=6pt,
  innertopmargin=6pt, innerbottommargin=0pt
]{promptframe}

\usetikzlibrary{calc}

\definecolor{RLBlue}{HTML}{2F6DB3}
\definecolor{BaseGray}{HTML}{D9DDE3}
\definecolor{BaseEdge}{HTML}{7F8A99}
\definecolor{AccentGreen}{HTML}{1E8E5A}
\definecolor{TitleText}{HTML}{1F2937}
\definecolor{BodyText}{HTML}{374151}
\definecolor{GuideGray}{HTML}{E5E7EB}

\usepackage{enumitem}

\usepackage{graphicx}
\newcommand{\cmark}{\ding{51}} 
\newcommand{\xmark}{\ding{55}} 

\usepackage{xspace}
\newcommand{\ourmethod}{InSight-doc\xspace}

\usepackage{makecell}

\usepackage{colortbl}
\makeatletter

\newcommand*{\@rowstyle}{}

\newcommand*{\rowstyle}[1]{
  \gdef\@rowstyle{#1}%
  \@rowstyle\ignorespaces%
}

\newcolumntype{=}{
  >{\gdef\@rowstyle{}}%
}

\newcolumntype{+}{
  >{\@rowstyle}%
}

\makeatother

\title{InSight-\emph{doc}: Agentic Visual Perception for Long-Document Understanding}

\author{%
    Kaican Li$^{1\ast}$ \qquad Weiyan Xie$^{1\ast}$ \qquad Lewei Yao$^2$ \qquad Jiannan Wu$^2$ \\ \textbf{Lanqing Hong$^2$ \qquad Yongxiang Huang$^2$ \qquad Nevin L. Zhang$^1$} \\ $^1$The Hong Kong University of Science and Technology \qquad $^2$Huawei \\
     \small{
       \textbf{Correspondence:} \texttt{\{klibf, wxieai, lzhang\}@cse.ust.hk} \qquad $^\ast$Equal contribution
     }
}

\begin{document}
\maketitle
\begin{abstract}
Long-document understanding often requires reasoning over many visually rich pages, making inference costly and prone to context rot.
In this work, we propose \emph{InSight-doc}, an agentic visual perception framework that treats visual resolution as an adaptive reasoning-time resource. InSight-doc starts from low resolution and selectively zooms into high-resolution regions for finer evidence, without relying on any external retriever. To train such an agent, we construct an active-perception corpus of 17.9K high-quality SFT examples with region-level zoom-in trajectories, accompanied by 19.2K hard RL examples. Through SFT+RL, InSight-doc-8B improves the baseline by 4.3--16.4 accuracy points over document VQA benchmarks.
On long documents, it reduces hallucination by more than 40\% and inference latency by 41\%--68\% while maintaining an accuracy lead.
Our code, datasets, and model are released at {\fontsize{9.8}{11}\selectfont \url{https://github.com/m-Just/InSight-doc}}.
\end{abstract}

\section{Introduction}
\label{sec:intro}

Frontier AI systems like multimodal large language models (MLLMs) are primarily based on transformers~\citep{vaswani2017attention} which use an attention mechanism to model complex data.
The attention mechanism, albeit powerful, requires the model to attend to all $N$ tokens in the context window.
This not only leads to an $O(N)$ space cost and an $O(N^2)$ time cost, but also underlies a phenomenon known as {\em context rot}~\citep{hong2025context} where models get worse dramatically as prompts get longer.
It is believed to be caused by diluted attention over long context and scarcity of genuinely long-context training data~\citep{anthropic2025effective}.
As a result, \emph{how to address long-context tasks \textbf{reliably and efficiently} has become a central question for AI research.}

\begin{figure}
    \centering
    \includegraphics[width=\columnwidth]{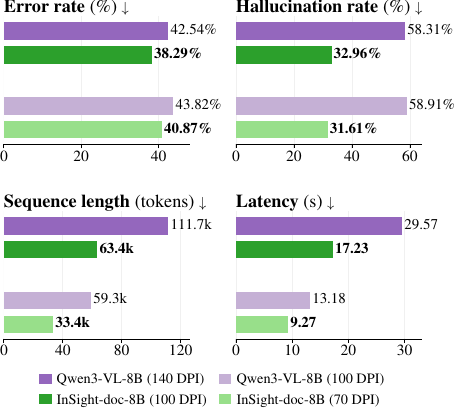}
    \caption{\textbf{InSight-doc substantially reduces hallucination, sequence length, and latency while improving accuracy on long-document VQA.} Hallucination rate is measured on unanswerable questions as the rate of non-abstaining answers. More details in Appendix~\ref{app:longdoc_efficiency_figure}.}
    \label{fig:longdoc_efficiency}
    \vspace{-1em}
\end{figure}

\begin{figure*}[t]
    \centering
       \includegraphics[width=\linewidth]{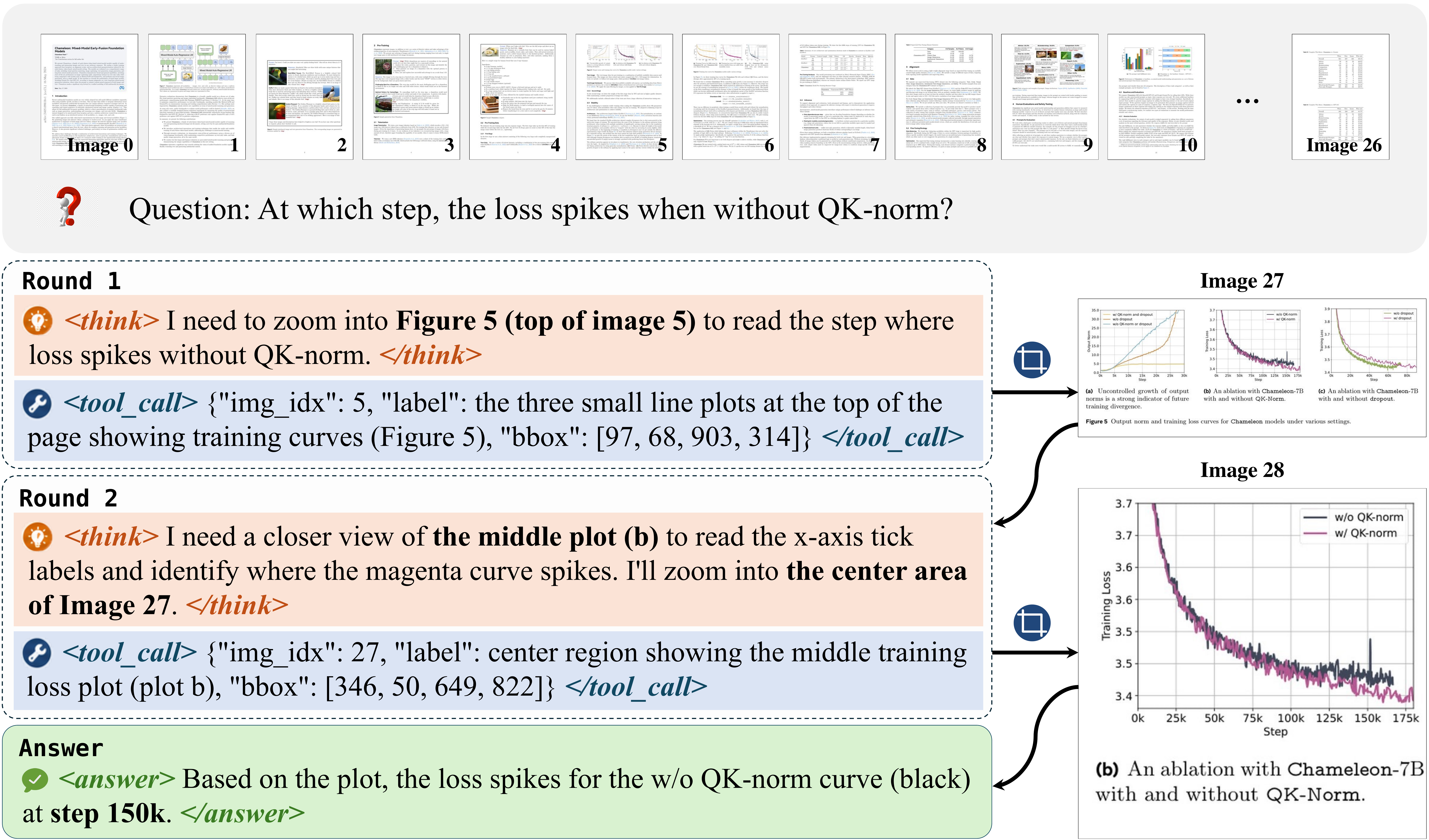}    
\caption{\textbf{Illustration of \ourmethod{} on a long-document VQA example.} Starting from a low-resolution overview, \ourmethod{} performs an interleaved multimodal chain-of-thought: each round emits a thought (\texttt{<think>}) and a zoom-in tool call (\texttt{<tool\_call>}) specifying \texttt{img\_idx}, \texttt{label}, and \texttt{bbox}; the cropped high-resolution region is then appended as visual evidence. After multiple rounds of active perception, the model produces the final \texttt{<answer>}. Images shown above are not in the exact scales seen by the model. More examples can be found in Appendix~\ref{app:example_appendix}.}
    \label{fig:teaser}
    \vspace{-0.5em}
\end{figure*}

Multi-page document understanding is a typical long-context task that is vital to many real-world applications of MLLMs.
The task usually involves answering questions over lengthy, visually rich documents such as research papers and financial reports~\citep{tito2023hierarchical, van2023document, ma2024mmlongbench, deng2025longdocurl}.
Traditional parsing-based methods use a multi-stage pipeline that involves layout detection, OCR, reading-order reconstruction, etc.~\citep{wang2024mineru,feng2025dolphin,dong2026qianfan}, which are prone to cascading errors and often do not generalize very well to visually-rich documents with complex layouts.
In contrast, \emph{vision-based} approaches utilizing MLLMs end-to-end~\citep{bai2025qwen2.5vl, zhu2025internvl3, hu2025mplug2} offer a more general, reliable solution by treating each document page as a normal image, preserving visual integrity and sidestepping fragile intermediate steps.

By default, documents are fed into MLLMs at a high resolution to minimize information loss, but this creates two issues for long documents: \emph{slow inference} due to quadratic computation cost, and \emph{degraded performance} due to context rot.
Apart from improving the compression rate of vision encoders~\citep[e.g.,][]{wei2025deepseek}, some \emph{coarse-to-fine} approaches that start from a downsampled document overview have been proposed recently.
\citet{xu2025cogdoc} propose a two-pass approach that first ``fast-reads'' a document under low resolution to identify relevant pages, and then shows the pages under high resolution in a new context for ``focused thinking''.
Doc-V$^\star$~\citep{zheng2026docvstar} extends the idea under the ReAct framework~\citep{yao2022react}, allowing the model to query external retrievers for relevant pages or pull pages directly by indices in a multi-turn interleaved fashion.


In this work, we propose {\textbf{\ourmethod}}, a fully \emph{end-to-end}, \emph{retriever-free}, agentic framework for reliable and efficient long document understanding.
\ourmethod moves beyond page retrieval into the ``thinking with images'' paradigm~\citep{openai2025o3} where tools like image cropping and zooming are used to enhance chain-of-thought (CoT) reasoning~\citep{su2025pixel,zheng2025deepeyes,lai2025mini,fan2025grit,zhang2025chain}.
Rather than passively processing documents at a fixed resolution or relying on external retrievers, \ourmethod empowers the model to \emph{dynamically seek, acquire, and integrate multi-scale visual evidence}.
As shown in Figure~\ref{fig:teaser}, \ourmethod begins with a low-resolution\footnote{E.g., 50 DPI where most text is still readable.} view of the full document and uses the zoom-in tool to look for the desired information.
The resulting visual evidence is directly appended to the reasoning chain for closer inspection.
This coarse-to-fine workflow mimics how humans normally read documents: \emph{start from a high-level overview and only jump into details if necessary}.
In this way, \ourmethod significantly reduces context pressure and allows the model to ``focus'' its attention on what is most likely relevant.

Our main contributions are as follows:

\begin{itemize}
    \item We propose InSight-doc, an \emph{end-to-end} agentic framework that adaptively acquires visual evidence during multi-round reasoning over long documents and visually-rich images.


    \item We construct a high-quality, diverse training corpus with 17.9K \emph{multi-hop} zoom-in SFT trajectories and 19.2K hard RL examples.

    \item Extensive experiments show that InSight-doc significantly improves the \emph{accuracy-efficiency Pareto frontier} over the baseline.
\end{itemize}

\section{Related Work}

\subsection{Document Understanding}

While MLLMs have achieved near-saturated performance on single-page benchmarks~\cite{mathew2021docvqa, mathew2022infographicvqa}, long-document understanding~\cite{tito2023hierarchical, ma2024mmlongbench, deng2025longdocurl} remains challenging. Existing methods can be categorized into three paradigms.

\noindent
\textbf{(1) End-to-end methods.} 
Most advanced models~\cite{bai2025qwen2.5vl, Qwen3-VL, zhu2025internvl3, wang2025internvl3.5, team2025kimivl, yang2025keyevl1.5, guo2025seed1.5vl} directly feed all high-resolution pages into MLLMs, producing a large number of visual tokens. Although this preserves fine-grained visual information, many pages may contain task-irrelevant content, leading to considerable computational overhead and limited scalability to extremely long documents.
\citet{xiong2025docr1,yan2026docseeker} identify/prioritize evidence within a fixed visual input and still do not allow the model to actively acquire new visual evidence.

\noindent
\textbf{(2) Visual retrieval-based methods.} 
To reduce visual tokens, recent studies have explored visual retrieval-augmented generation (RAG) methods~\cite{yu2024visrag,cho2024m3docrag,faysse2024colpali,chen2024sv,tanaka2025vdocrag,ma2024dse, wang2026vrag, wu2025doc}. These methods use an external retriever to align textual queries with document images and select the top-$k$ pages by embedding similarity. While efficient, this decoupled pipeline can be sensitive to the choice of $k$ and may struggle when crucial evidence lies outside the retrieved subset.

\noindent
\textbf{(3) Coarse-to-fine methods.} 
To balance scalability and fidelity, recent studies~\cite{xu2025cogdoc, zheng2026docvstar} adopt a coarse-to-fine strategy: they first scan the full document at low resolution to identify relevant pages, and then inspect them at high resolution. 
\ourmethod follows this line of work but differs in two key aspects. First, \ourmethod performs region cropping and zooming rather than page fetching.
This enables \textit{sub-page}, \textit{region-level} grounding within a unified multimodal chain-of-thought while reducing token overhead.
Second, its \emph{retriever-free}, interleaved region-text reasoning naturally supports multi-hop queries and multi-turn self-correction.
While Doc-V$^\star$~\citep{zheng2026docvstar} also supports similar interleaved reasoning, it primarily relies on an external retriever to fetch pages.
This introduces extra indexing/retrieval costs and errors from the retriever side.

\subsection{Visual Search}

Visual search requires models to actively perceive fine-grained regions of interest and perform region-text interleaved reasoning, rather than interpreting the whole image at a fixed scale. Research in this field has evolved from external detectors and scripted zoom workflows~\citep{vstar,shen2024zoomeye,li2025dyfo} to ``think with images'' MLLMs that internalize zoom/crop operations through reinforcement learning or curated trajectories~\citep{openai2025o3,zheng2025deepeyes,su2025pixel,lai2025mini,li2026insighto}. We discuss these methods in detail in Appendix~\ref{app:related_work_vs}.

Despite this progress, existing visual search systems are mainly evaluated on natural photographs or single-page, text-rich images~\citep{vstar,mme_rw,wang2025traceable,lai2025mini, wang2025simple-o3}, where models usually locate a single salient region within one image. Multi-page documents, however, require searching across a \emph{collection} of pages, with evidence often scattered across distant and non-adjacent pages.
MLLMs that can jointly reason and dynamically acquire \emph{multi-region} visual evidence under a strict token budget remain largely underexplored.

\section{\ourmethod}
\label{sec:method}

\ourmethod targets \textbf{long-document VQA} where given a document \(\mathcal{D} = \{p_i\}_{i=1}^N\) of \(N\) pages (as images), it answers a user query about the document.
The answer must be well supported by the document; otherwise, it should clearly indicate that the query is unanswerable.



The basic approach is to feed the document and the query into an instruction-tuned MLLM, which would typically generate a chain of thought (CoT) carrying its reasoning process and then provide its answer within a single turn~\citep{wei2022chain,kojima2022large}.
This approach implicitly assumes that the given images are fixed.
\textbf{\ourmethod} drops this assumption by enabling the MLLM to zoom into the images inside a reasoning-action loop~\citep{yao2022react}.
Crucially, the zoom-in operation is carried out on a high-resolution version of the document (if available), so we can start with a low-resolution version of the document as the initial input, and the model is  still able to \emph{fully} recover the lost information later on.
This enables more aggressive downsampling (and thus greatly shortening the context) without needing to worry much about information loss.


\subsection{Implementation}
\label{sec:impl}
Formally, \ourmethod initializes its visual context space \(\mathcal{I}_{\text{ctx}}^{(0)}\) (i.e., the set of images available to the model at reasoning step \(t=0\)) with the initial page images $\{\tilde{I}^{(0)}_k\}_{k=1}^N$.
These images are usually downsampled from the original high-resolution images $\{I_k\}_{k=1}^N$ through $\tilde{I}^{(0)}_k = \texttt{resize}(I_k, r \cdot \texttt{size}(I_k))$
where $r \leq 1$ is the initial resize factor.
Then, based on this initial input, the model decides (via CoT) whether it needs to zoom into a certain region for a clearer view or more information.

At any step \(t\), if the model decides to zoom in, it emits a tool call: $\texttt{zoom\_in}(k, d, b \mid \mathcal{I}_{\text{ctx}}^{(t-1)})$
where \(k\) is the index of the target image within the current visual context \(\mathcal{I}_{\text{ctx}}^{(t-1)}\), \(d\) is a free-form natural language description of the region of interest, and $b$ the bounding box of the region.
If the tool call is valid, the requested region is then cropped from the corresponding high-resolution source $I_{s(k)}$, yielding $I^{(t)}_\text{crop} = \texttt{crop}(I_{s(k)}, b, r)$
which is then (optionally) resized to an appropriate scale by
\begin{equation}
    \tilde{I}^{(t)}_\text{crop} = \texttt{resize}(I^{(t)}_\text{crop}, c \cdot r \cdot \texttt{size}(I^{(t)}_\text{crop})),
\end{equation}
where $c > 1$ can be seen as the zoom factor with respect to the low-resolution image $\tilde{I}^{(t)}_k$.
For recursive zoom-ins, the resize factor associated with the new crop is updated as $r \gets c \cdot r$.
The crop at time $t$ is appended to the visual context space: $\mathcal{I}_{\text{ctx}}^{(t)} = \mathcal{I}_{\text{ctx}}^{(t-1)} \cup \{\tilde{I}^{(t)}_\text{crop}\}.$
This repeats until the model decides to answer or the tool limit is reached.



\subsection{Inference Cost Analysis}

Although image resizing reduces the initial visual-token budget, our approach may introduce additional overhead with longer multi-turn interactions.
To characterize this tradeoff, we analyze both the total sequence length and inference latency relative to a single-turn, no-resize baseline.

Let $P$ and $R$ denote the numbers of input and generated tokens, respectively.
For long-context, full-attention inference, we approximate latency by
\begin{equation}
\label{eq:single_turn_cost}
T(P,R)=\alpha P^2+\beta R(2P+R),
\end{equation}
where $\alpha,\beta>0$ capture prefill and decoding costs.
Fixed and lower-order terms are omitted.
Minor overheads such as image encoding are also omitted.

Let $P_0$ and $R_0$ be the input- and generated-token counts of the
single-turn, no-resize baseline. For an image
side-length resize ratio $r\in(0,1]$, let $n(r)$ be the number of
zoom-in tool calls.
Assuming that visual tokens dominate the baseline input, we model the total input and output lengths of our approach as $P_r=x(r)P_0$ and
$R_r=y(r)R_0$, where
\begin{equation}
\label{eq:normalized_token_model}
x(r)=r^2+\delta n(r),
\quad
y(r)=1+\lambda n(r).
\end{equation}
The parameters $\delta$ and $\lambda$ denote the
input- and output-token costs of one tool call relative to $P_0$ and $R_0$, respectively.
We further define the baseline prompt-to-response ratio
$\kappa=P_0/R_0$ and prefill-to-decoding coefficient ratio $\gamma=\alpha/\beta$.

\begin{proposition}[Relative sequence length]
\label{prop:relative_sequence_length}
Let $S_0=P_0+R_0$ and $S_r=P_r+R_r$ be the total sequence lengths of
the baseline and our approach, respectively.
The relative sequence length satisfies
\begin{equation}
\label{eq:relative_sequence_length}
S_r/S_0
\leq x(r) + \kappa^{-1}y(r).
\end{equation}
\end{proposition}

The proof is straightforward (see Appendix~\ref{app:proof-relative-seq-len}).
For long-document VQA tasks, consider
$r\in[0.25,0.50]$, $n(r)\in[1,3]$, $\delta\in[0.01,0.05]$,
$\lambda\in[0.10,0.50]$, and $\kappa\in[50,200]$. Under these ranges,
the upper bound in Proposition~\ref{prop:relative_sequence_length} ranges from
$7.8\%$ to $45.0\%$ of the baseline sequence length.

\begin{proposition}[Relative latency]
\label{prop:relative_latency}
Let $T_0$ and $T_r$ denote the prefill and decoding latencies of the baseline and our approach, respectively, where $T_0=T(P_0,R_0)$. With prefix caching and $\beta\geq\alpha$, the relative latency satisfies
\begin{equation}
\label{eq:relative_latency_bound}
T_r/T_0
\leq
w_{\mathrm{p}}x(r)^2
+w_{\mathrm{c}}x(r)y(r)
+w_{\mathrm{g}}y(r)^2,
\end{equation}
where the weights are defined by
\begin{equation}
(w_{\mathrm{p}}, w_{\mathrm{c}}, w_{\mathrm{g}})
=
\frac{(\gamma\kappa^2, 2\kappa, 1)}{\gamma\kappa^2+2\kappa+1}.
\end{equation}
\end{proposition}

The proof is provided in Appendix~\ref{app:proof-relative-latency}.
Consider the representative values $\gamma=10^{-2}$, $\kappa=100$, $\delta=0.05$, and $\lambda=0.5$.
Table~\ref{tab:relative_latency} reports the resulting upper bounds.
For example, at $r=0.35$ with two tool calls (i.e., $n=2$), the resulting latency is bounded by ${\sim}32.5\%$ of the no-resize baseline.

\begin{table}[t]
\centering
\setlength{\tabcolsep}{7.4pt}
\fontsize{9}{11}\selectfont
\caption{Relative-latency upper bound for representative parameters.
Here, $n$ is the number of zoom-in tool calls.}
\label{tab:relative_latency}
\begin{tabular}{cccccc}
\toprule
$r$ & $n=0$ & $n=1$ & $n=2$ & $n=3$ & $n=4$ \\
\midrule
$0.25$ & $0.046$ & $0.124$ & $0.238$ & $0.389$ & $0.576$ \\
$0.35$ & $0.090$ & $0.189$ & $0.325$ & $0.498$ & $0.707$ \\
$0.50$ & $0.190$ & $0.336$ & $0.519$ & $0.738$ & $0.994$ \\
\bottomrule
\end{tabular}
\end{table}

\section{Data Construction}
\label{sec:data}

To train our model to reason from low-resolution inputs and iteratively decide \emph{where to zoom in}, we curate a large-scale corpus of document QA data that is simultaneously \textbf{multi-source}, \textbf{multi-page}, \textbf{multi-hop}, while covering \textbf{multiple question types}, paired with explicit zoom-in chain-of-thought (CoT) trajectories. Figure~\ref{fig:data_create} summarizes the full pipeline: a three-stage filtering and CoT-construction process (top) that routes items either to SFT or to RL, and an \textbf{InSight-o3}-based~\citep{li2026insighto} two-agent trajectory generator (bottom) whose output is merged into a single flat multimodal CoT used as the imitation target.

\begin{figure*}[t]
    \centering
     \includegraphics[width=0.99\linewidth]{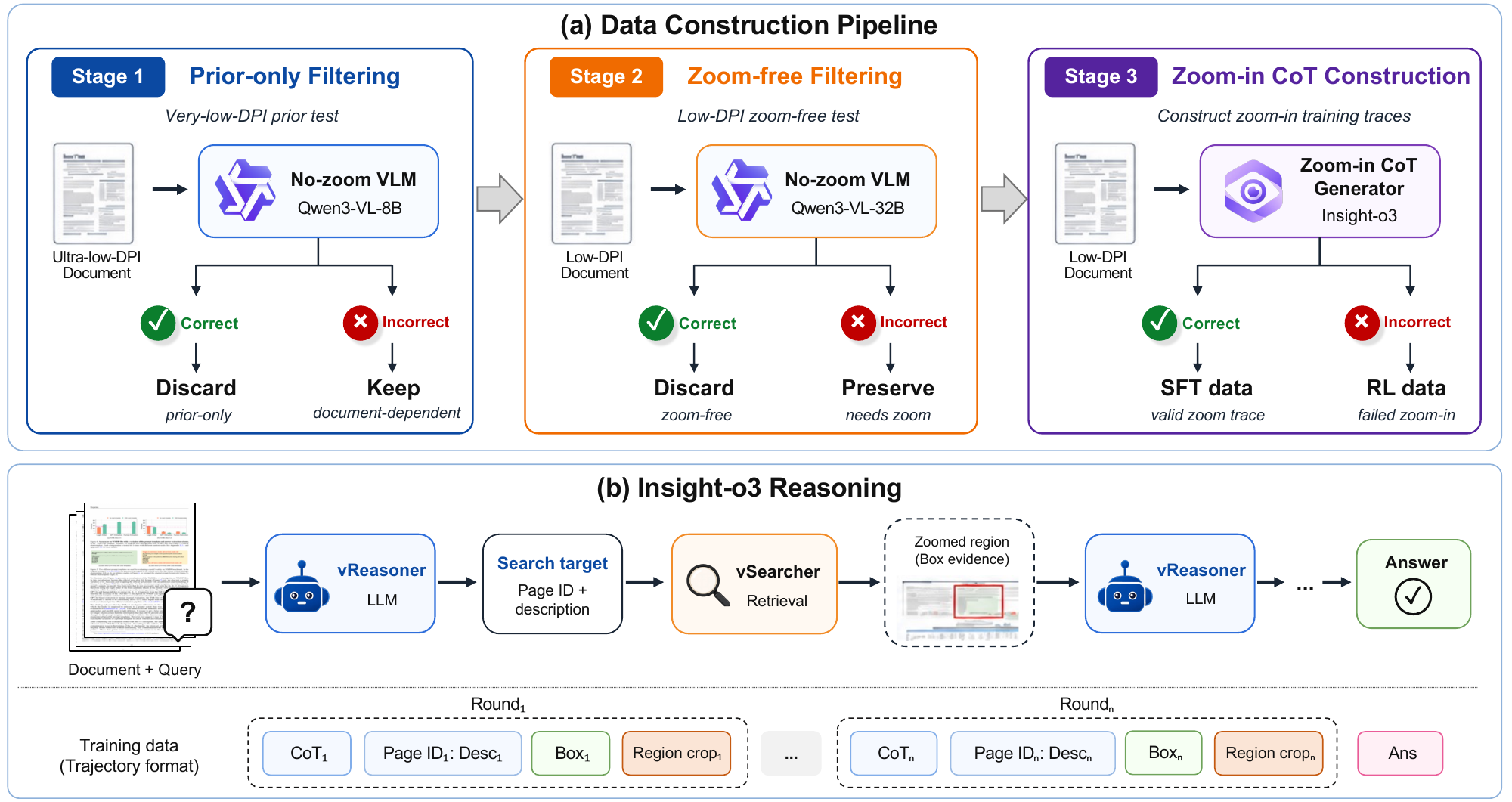}
   \caption{
\textbf{Overview of data construction pipeline.}
The \textbf{top panel} shows a three-stage filtering and CoT construction process: Stage~1 discards questions answerable without the document, Stage~2 discards questions already answerable at low DPI without zoom, and Stage~3 uses InSight-o3 to construct zoom-in CoTs for the remaining items, routing successes to SFT and failures to RL.
The \textbf{bottom panel} illustrates the InSight-o3 two-agent trajectory, where the vReasoner and vSearcher iteratively produce reasoning steps, zoom-in requests, bounding boxes, and crops, which are then merged into a single flat multimodal CoT used as the imitation target for \ourmethod.
}
    \label{fig:data_create}
\end{figure*}

\subsection{QA Generation}
\label{sec:qagen}

\paragraph{Document sources \& single-hop QAs.}
We curate documents from six complementary sources: \textbf{arXiv}, \textbf{DUDE}~\citep{van2023document}, \textbf{DocVQA}~\citep{mathew2021docvqa}, \textbf{InfographicVQA}~\citep{mathew2022infographicvqa}, \textbf{Paper2Poster}~\citep{pang2026paperposter}, and \textbf{MapTab}~\citep{shang2026maptab}.
For arXiv, we use \emph{MinerU}~\citep{wang2024mineru} to extract visual elements from PDFs and prompt Gemini 3.1 Flash-Lite~\citep{google2025gemini31flashlite} to generate enriched descriptions of these visuals and synthesize QA pairs grounded in both the visuals and their descriptions. For the remaining sources, we reuse curated subsets of their native QA pairs.

\paragraph{Multi-page and multi-hop construction.} 
To encourage long-document, multi-region reasoning, we build longer-context and multi-hop variants from single-hop QAs. For \emph{multi-page} construction, we either expand a DocVQA question to its full source document or merge the evidence-bearing document with other documents from the same source while preserving page order. For \emph{multi-hop} construction, we use two strategies: (1) \textbf{visual-grounded synthesis} on arXiv, which samples \(k\in\{2,3,4\}\) visuals from the same paper to generate comparison or aggregation questions, and (2) \textbf{single-hop QA merging}, which combines multiple single-hop QAs into one compound item requiring spatially separated evidence. Details are provided in Appendix~\ref{app:data_details}.


\subsection{Filtering and Zoom-in CoT Generation}
\label{sec:filter_cot}

Not every QA teaches active perception. We therefore apply a three-stage cascade (Figure~\ref{fig:data_create}, top) to filter and partition the data:

\begin{enumerate}
    \item \textbf{Prior-only filtering.} Questions answered correctly at ultra-low resolution (20 DPI) by Qwen3-VL-8B~\citep{Qwen3-VL} are deemed document-independent and discarded.
    \item \textbf{Zoom-free filtering.} Surviving QAs are rendered at DPI \(\in\{50,70,100\}\). Items answered correctly by Qwen3-VL-32B without zoom are already legible and thus removed.
    \item \textbf{Zoom-in CoT construction.} For the remaining QAs, we leverage the \textbf{InSight-o3}~\citep{li2026insighto} pipeline to synthesize explicit active-perception trajectories, iterating until a final answer is produced.
\end{enumerate}

Finally, items whose CoT yields the correct answer, judged by GPT-5-nano~\citep{openai2026gpt5systemcard}, become \textbf{SFT} data; the rest are reserved for \textbf{RL}.

\paragraph{Multimodal CoT.} To build the CoT trajectories, InSight-o3 employs a two-agent setup: a \emph{vReasoner} maintains the reasoning state and emits structured requests \(\langle\text{PageID}_i{:}\text{Desc}_i\rangle\) identifying regions that warrant a closer look, while a \emph{vSearcher} localizes the described region on page \(\text{PageID}_i\), returning a bounding box \(\langle\text{Box}_i\rangle\); the corresponding image patch within \(\langle\text{Box}_i\rangle\) on \(\text{PageID}_i\) is then cropped as \(\text{crop}_i\) and appended to the context. We merge the two-agent trajectory into one flat multimodal sequence (lower track of Figure~\ref{fig:data_create}), which serves directly as the SFT target.
This allows the model to jointly learn \emph{when} to zoom in, \emph{what} region to request, and \emph{how} to integrate the evidence, replacing the external vSearcher with its own localization predictions.
We use GPT-5-mini~\citep{openai2026gpt5systemcard} as the vReasoner and a fine-tuned Qwen3-VL-8B-Instruct~\citep{Qwen3-VL} as the vSearcher.
More details can be found in Appendix~\ref{app:data_details}.

\subsection{Dataset Properties}
\label{sec:data_props}

Our final training corpus contains 37,149 QA instances. The SFT split includes 17,913 trajectories, of which 14,216 are answerable (79.36\%) and 3,697 are unanswerable (20.64\%). The RL split includes 19,236 prompts, of which 10,579 are answerable (55.00\%). The average document length is 18.51 pages overall, and 17.75 pages for the SFT split. The SFT trajectories contain 2.61 assistant CoT/tool-use rounds on average. More detailed statistics are provided in Appendix~\ref{app:dataset-statistics}.

\section{Experiments}

\subsection{Setup}

\paragraph{Training configuration.}
We use Qwen3-VL-8B-Instruct~\citep{Qwen3-VL} as the base model for SFT and subsequent RL.
We use GRPO~\cite{shao2024deepseekmath} as the RL algorithm and we only use a binary accuracy reward for RL.
The hyperparameter settings are provided in Appendix~\ref{app:hparam}.
Our SFT and RL code is based on \emph{verl}~\citep{sheng2024hybridflow}.
We use weighted sampling for RL (see Appendix~\ref{app:weighted-sampling}).

\paragraph{Evaluation setting.}
We evaluate on the following three types of benchmarks:
\begin{itemize}[leftmargin=3mm]
    \item \textbf{Standard document VQA}: DUDE \citep{van2023document} and MP-DocVQA \citep{tito2023hierarchical}; with 5.7 and 7.0 pages on average.
    \item \textbf{Long document VQA}: MMLongBench-Doc \citep{ma2024mmlongbench} and LongDocURL \citep{deng2025longdocurl}; with 49.4 and 85.6 pages on average.
    \item \textbf{General high-res{.} VQA}: MME-RealWorld-Lite \citep{mme_rw} and O3-Bench \citep{li2026insighto}; both are single-image based.
\end{itemize}

More benchmark details are in Appendix~\ref{app:eval-data}.
All PDF documents are rasterized at 200 DPI as ground-truth high-resolution images.
They are then downsampled by $r \in \{0.25, 0.35, 0.5, 0.7\}$, which are equivalent to DPIs of $\{50, 70, 100, 140\}$.\footnote{
For reference, the resolution of a typical document page at 50 DPI ($r=0.25$) is roughly $425 \times 550$ pixels.}
For documents that cannot fit within the default maximum context length of a model, we downsample the images by $50\%$ (area-wise) up to four times until they fit or fail.
MME-RealWorld-Lite and O3-Bench are image-based so there is no rasterization but the resize still applies.
We use GPT-5-nano as the judge model for answer correctness.
The judge is calibrated with a manually labeled test set (see Appendix~\ref{app:judge-calibration}).

\subsection{Main Results}

\begin{table*}
\centering
\caption{\textbf{Comparison with frontier models on medium-to-long document VQA.} The models are compared under two initial resolution settings: \textit{low} ($r = 0.25$, DPI 50) and \textit{medium} ($r = 0.5$, DPI 100).
The inference mode ``E2E'' means single-turn QA without tool use, while ``Agent'' allows multi-turn zoom-in tool use.
LLM-as-judge accuracy is reported on all benchmarks.
Input page limit is set to 40 for all models due to API limit on some closed models. The limit only affects MMLongBench-Doc and LongDocURL.
The unlimited results can be found in Appendix~\ref{app:uncapped_longdoc_eval}.}
\label{tab:sp_results1}
\setlength{\tabcolsep}{7.4pt}
\fontsize{9}{11}\selectfont
\begin{tabular}{l c cc cc cc cc cc}
\toprule
& & \multicolumn{2}{c}{\textbf{DUDE}} & \multicolumn{2}{c}
{\textbf{MP-DVQA}} & \multicolumn{2}{c}{\textbf{MMLong.}} &
\multicolumn{2}{c}{\textbf{LongDoc.}} & \multicolumn{2}{c}
{\textbf{Average}} \\
\cmidrule(lr){3-4} \cmidrule(lr){5-6} \cmidrule(lr){7-8}
\cmidrule(lr){9-10} \cmidrule(lr){11-12}
& & 0.25 & 0.5 & 0.25 & 0.5 & 0.25 & 0.5 & 0.25 & 0.5 & 0.25 & 0.5
\\

\hline
\textcolor{gray}{\fontsize{7}{8}\textit{Closed proprietary
models}} \\
GPT-5.4-nano          & E2E & 52.8 & 66.0 & 64.2 & 83.4 & 33.6 & 52.2 &
54.4 & 72.7 & 51.2 & 68.6 \\
GPT-5.4-mini          & E2E & 63.2 & 71.6 & 78.3 & 88.9 & 46.4 & 58.8 &
69.3 & 77.6 & 64.3 & 74.2 \\
GPT-5-mini            & E2E & 63.9 & 70.0 & 81.3 & 89.5 & 48.0 & 57.2 &
70.3 & 80.4 & 65.9 & 74.3 \\
Gemini-3.1-Flash-lite & E2E & 70.2 & 70.8 & 87.5 & 89.4 & 57.4 & 58.3 &
75.8 & 75.7 & 72.7 & 73.5 \\
Gemini-3-Flash        & E2E & 72.0 & 72.4 & 89.3 & 89.9 & 62.4 & 61.8 &
77.6 & 76.3 & 75.3 & 75.1 \\
\hline
\hline
\textcolor{gray}{\fontsize{7}{8}\textit{Open models}} \\
InternVL3-8B              & E2E & 52.7 & 62.2 & 66.5 & 84.7 & 13.6 &
33.5 & 26.7 & 42.9 & 39.9 & 55.8 \\
GLM-4.6V-Flash (9B)       & E2E & 40.7 & 53.4 & 50.9 & 72.5 & 14.8 &
15.0 & 24.6 & 26.0 & 32.8 & 41.7 \\
Qwen3-VL-8B               & E2E & 52.9 & 68.5 & 65.1 & 84.9 & 33.7 &
51.4 & 50.5 & 68.4 & 50.5 & 68.3 \\
Qwen3-VL-8B (w/ zoom)    & Agent & 55.1 & 66.4 & 66.1 & 82.6 & 33.2 &
48.8 & 47.1 & 63.8 & 50.4 & 65.4 \\

\hline
\textcolor{gray}{\fontsize{7}{8}\textit{InSight-doc models}} \\
\ourmethod-8B (SFT)    & Agent & 60.8 & 67.2 & 72.2 & 79.5 & 36.5 & 48.0
& 57.0 & 63.7 & 56.6 & 64.6 \\
\rowcolor{cyan!8}
\ourmethod-8B (SFT+RL) & Agent & 70.1 & 73.8 & 83.4 & 87.6 & 50.8 & 58.6
& 63.3 & 70.5 & 66.9 & 72.6 \\
$\quad\Delta$ \textit{w.r.t. Qwen3-VL-8B} & & \textit{\llap{+}17.2} & \textit{\llap{+}\hphantom{0}5.3} & \textit{\llap{+}18.3} & \textit{\llap{+}\hphantom{0}2.7} & \textit{\llap{+}17.1} & \textit{\llap{+}\hphantom{0}7.2} & \textit{\llap{+}12.8} & \textit{\llap{+}\hphantom{0}2.1} & \textit{\llap{+}16.4} & \textit{\llap{+}\hphantom{0}4.3} \\
\bottomrule
\end{tabular}
\end{table*}

\begin{table}
\centering
\caption{\textbf{Comparison with frontier models on general high-resolution VQA.}
LLM-as-judge
accuracy is reported under two initial resolution settings:
$r=0.25$ and $r=0.5$.
The inference modes follow Table~\ref{tab:sp_results1}.}
\label{tab:mmlite_o3bench_results}
\setlength{\tabcolsep}{5.5pt}
\fontsize{9}{11}\selectfont
\begin{tabular}{l cc cc}
\toprule    
& \multicolumn{2}{c}{\textbf{MME-RW}$_\ell$} & \multicolumn{2}{c}
{\textbf{O3-Bench}} \\
\cmidrule(lr){2-3} \cmidrule(lr){4-5}
& 0.25 & 0.5 & 0.25 & 0.5 \\
\midrule
GPT-5.4-nano          & 40.3 & 47.9 & 29.6 & 32.8 \\
GPT-5.4-mini          & 47.4 & 57.2 & 40.9 & 56.2 \\
Gemini-3.1-Flash-lite & 47.8 & 51.8 & 44.6 & 51.0 \\
\midrule
Qwen3-VL-8B              & 41.0 & 49.2 & 22.3 & 35.1 \\
Qwen3-VL-8B (w/ zoom) & 42.5 & 50.6 & 20.3 & 35.4 \\
\rowcolor{cyan!8}
\ourmethod-8B (SFT+RL) & 48.2 & 52.9 & 24.1 & 43.8 \\

\bottomrule
\end{tabular}
\end{table}

\paragraph{Document VQA.}
Table~\ref{tab:sp_results1} compares \ourmethod with open and proprietary models on two medium and two long document VQA benchmarks.
At the low initial resolution ($r=0.25$), \ourmethod-8B (SFT+RL) achieves an average accuracy of 66.9\%, significantly outperforming its base model, Qwen3-VL-8B, by \textbf{16.4 points}.
The gains are consistent across all four benchmarks: 17.2 points on DUDE, 18.3 points on MP-DocVQA, 17.1 points on MMLongBench-Doc, and 12.8 points on LongDocURL.
At the medium resolution ($r=0.5$), \ourmethod reaches an average accuracy of 72.6\%, improving over Qwen3-VL-8B by \textbf{4.3 points}.
RL contributes substantially beyond SFT alone, increasing the average accuracy from 56.6\% to 66.9\% at $r=0.25$ and from 64.6\% to 72.6\% at $r=0.5$.
Despite using an open 8B backbone, \ourmethod also remains competitive with proprietary models: at $r=0.25$, it outperforms all the GPT variants, and at $r=0.5$, it is still comparable with both the GPT and the Gemini models.

\paragraph{General high-resolution VQA.}
Table~\ref{tab:mmlite_o3bench_results} evaluates whether the learned policy generalizes beyond document VQA to general high-resolution VQA.
Relative to Qwen3-VL-8B without zoom, \ourmethod improves MME-RealWorld-Lite by 7.2 points at $r=0.25$ and 3.7 points at $r=0.5$, while improving O3-Bench by 1.8 and 8.7 points, respectively.
It also consistently outperforms the Qwen3-VL-8B zoom agent, showing that access to a zoom tool alone is insufficient without an appropriately trained  policy.
On MME-RealWorld-Lite at $r=0.25$, \ourmethod obtains 48.2\%, slightly surpassing the best proprietary result of 47.8\%.
On O3-Bench at $r=0.5$, it reaches 43.8\%, surpassing GPT-5.4-nano, although it remains below GPT-5.4-mini and Gemini-3.1-Flash-lite.

\paragraph{Comparison with related methods.}
Table~\ref{tab:closest_related} compares InSight-doc with existing visual-retrieval and coarse-to-fine document reasoning methods.
InSight-doc is the only method in the table that jointly supports retriever-free, coarse-to-fine, and iterative evidence acquisition.
It achieves 57.8\% on MMLongBench-Doc and 65.6\% on LongDocURL, exceeding the strongest previously reported results by 15.7 and 9.3 points, respectively.
Since these are cross-paper results with differences in backbones, training data, input resolution, evaluation protocols, etc., they should be interpreted as contextual rather than fully controlled comparisons.
See Appendix~\ref{app:comp-w-related-methods} for a more comprehensive discussion.

\subsection{Performance Analysis}

\begin{figure*}[h]
    \centering  
     \includegraphics[width=\linewidth]{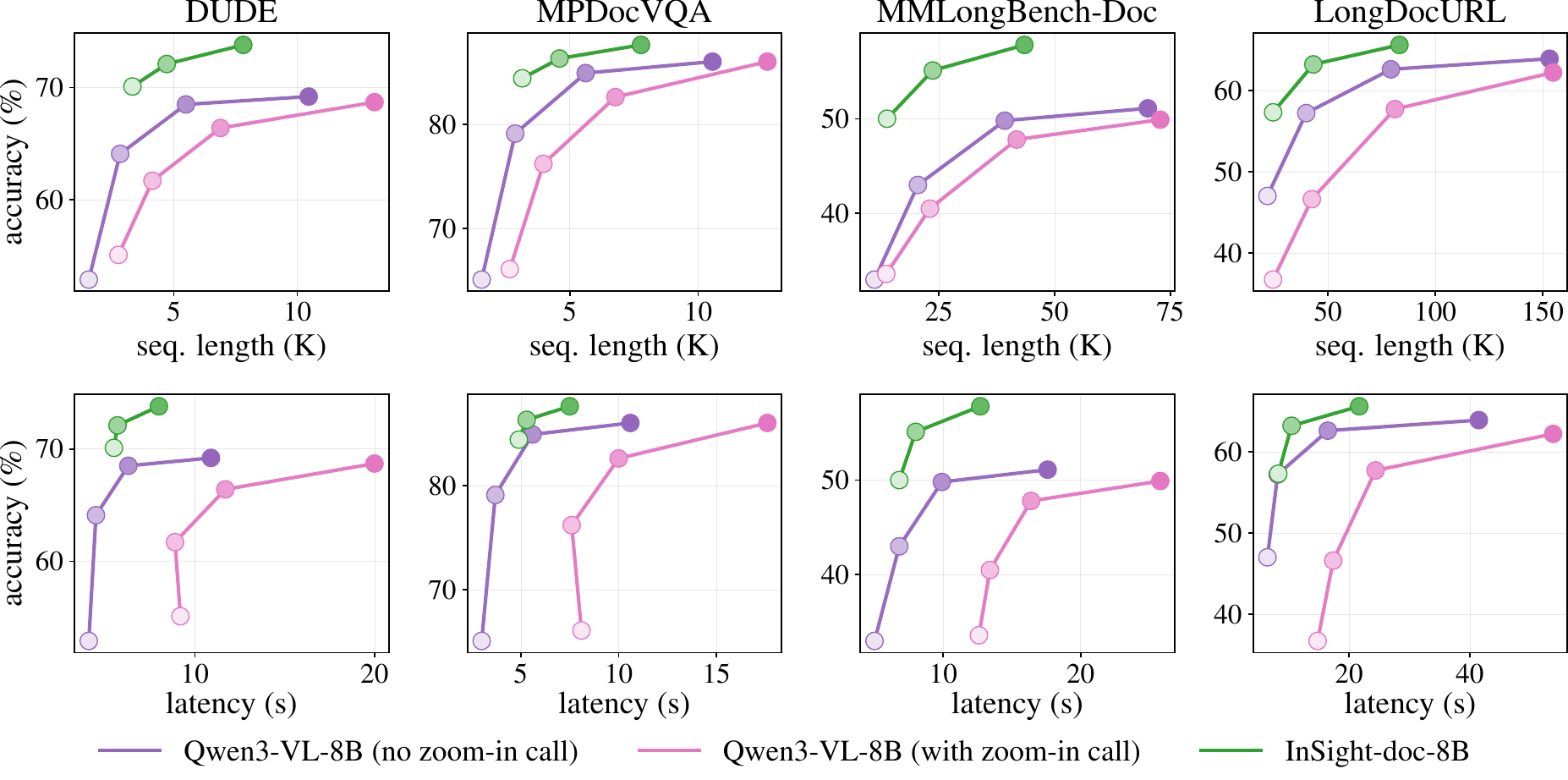}    
\caption{
\textbf{Accuracy vs.\ efficiency on four multi-page document VQA benchmarks.} We compare \textbf{InSight-doc} (ours, green) with Qwen3-VL-8B without zoom-in calls (purple) and with zoom-in calls (pink). For each method, darker shades indicate larger input resize ratios (higher resolution, more visual tokens). Notably, InSight-doc shows a clear tendency to push the Pareto frontier toward the upper-left, achieving a more favorable trade-off.
}
    \label{fig:expriment_main1}

\end{figure*}

As shown in Figure~\ref{fig:expriment_main1}, InSight-doc-8B \emph{dominates} Qwen3-VL-8B on \emph{accuracy-efficiency Pareto frontiers}, often achieving \textbf{higher accuracy at much lower sequence length and end-to-end latency}.


\paragraph{Sequence length reduction.}
Autoregressive decoding under long context often consumes a large amount of KV cache.
Sequence length can be seen as a proxy for \emph{context-related memory consumption}, mainly the KV cache.
Across all four benchmarks, \ourmethod achieves comparable or higher accuracy with substantially shorter sequences than Qwen3-VL-8B.
At 50 DPI ($r = 0.25$), \ourmethod obtains 66.9\% average accuracy, nearly matching the 68.3\% of the 100-DPI ($r = 0.5$) baseline while using \textbf{58\% fewer tokens}.
At 70 DPI ($r = 0.35$), it reaches 70.6\%, exceeding the 69.4\% of the 140-DPI baseline while \textbf{reducing the token count by 66\%}.
The reduction is even more pronounced on the longest-document examples (Figure~\ref{fig:expriment_longest}): \ourmethod at 70 DPI achieves 56.2\% accuracy using 42.4k tokens, compared with 53.2\% accuracy and 136.8k tokens for the 140-DPI baseline, corresponding to a \textbf{69\% reduction}, consistent with Proposition~\ref{prop:relative_sequence_length}.

\paragraph{Latency reduction.}
The shorter contexts also translate into lower inference latency.
At 50 DPI, \ourmethod approaches the accuracy of the 100-DPI baseline while \textbf{reducing average latency by 16\%}.
At 70 DPI, it outperforms the 140-DPI baseline by 1.2 points on average while \textbf{reducing latency by 54\%}.
For example, on MMLongBench-Doc, \ourmethod achieves 55.6\% accuracy in $9.3\,\mathrm{s}$, compared with 52.0\% in $21.2\,\mathrm{s}$ for the 140-DPI baseline.
On LongDocURL, it matches the baseline's 70.5\% accuracy using $11.0\,\mathrm{s}$ instead of $17.3\,\mathrm{s}$.
On the longest-document subset (Figure~\ref{fig:expriment_longest}), \ourmethod at 70 DPI requires only $11.2\,\mathrm{s}$ per example, compared with $39.3\,\mathrm{s}$ for the 140-DPI baseline, yielding a \textbf{71\% reduction} while improving accuracy by 3.0 points.
These reduction rates are again consistent with a theoretical prediction of about 48\%--81\% (Proposition~\ref{prop:relative_latency}).

Overall, the results suggest that \emph{\ourmethod mitigates context rot by bringing relevant context under close inspection while significantly reducing space and time cost}.

\subsection{Unanswerable Questions}

\begin{table}[t]
  \centering
  \caption{\textbf{Not-answerable F1 scores} under low and medium initial-input resolutions ($r=0.25$ and $0.5$).}
  \label{tab:not_answerable_f1}
  \resizebox{\linewidth}{!}{%
      \begin{tabular}{lcccc}
          \toprule
          & \multicolumn{2}{c}{\textbf{DUDE}} & \multicolumn{2}{c}{\textbf{MMLong.}} \\
          \cmidrule(lr){2-3} \cmidrule(lr){4-5}
          & 0.25 & 0.5 & 0.25 & 0.5 \\
          \midrule
          Qwen3-VL-8B              & 44.5 & 57.4 & 48.5 & 55.9 \\
          Qwen3-VL-8B (w/ zoom)    & 50.7 & 55.6 & 58.7 & 62.7 \\
          \ourmethod-8B (SFT+RL)   & 69.1 & 72.4 & 74.4 & 75.1 \\
          \bottomrule
      \end{tabular}%
  }
\end{table}

\begin{table}[t]
\centering
\fontsize{9}{11}\selectfont
\setlength{\tabcolsep}{4.3pt}
\caption{
\textbf{Comparison with visual-retrieval and coarse-to-fine methods.}
\textbf{R-f}: retriever-free, \textbf{C2F}: coarse-to-fine, \textbf{Itr}: iterative multi-turn evidence acquisition, \textbf{Rgn}: region-level evidence, \textbf{MMLD.}: MMLongBench-Doc, and \textbf{LDoc.}: LongDocURL.
Scores are official cross-paper numbers and are \emph{not fully controlled} for backbone, training data, input resolution, page budget, or evaluation protocol. $^\dagger$With GPT-4o as the main agent; other methods are based on open models of similar sizes.
}
\label{tab:closest_related}
\begin{tabular}{lcccccc}
\toprule
\textbf{Method} &
\textbf{R-f} &
\textbf{C2F} &
\textbf{Itr} &
\textbf{Rgn} &
\textbf{MMLD.} &
\textbf{LDoc.} \\
\midrule
ColPali$^\dagger$& \xmark & \xmark & \xmark & \xmark & 30.8 & -- \\
Doc-React$^\dagger$        & \xmark & \xmark & \cmark & \xmark & 38.3 & -- \\
VDocRAG          & \xmark & \xmark & \xmark & \xmark & 18.4  & 39.8 \\
VRAG-RL          & \xmark & \cmark & \cmark & \cmark & 26.6  & 44.9 \\
CogDoc           & \cmark & \cmark & \xmark & \xmark & 33.0  & -- \\
DocSeeker        & \cmark & \xmark & \xmark & \xmark & 40.1  & 51.7 \\
Doc-V$^{\star}$      & \xmark & \cmark & \cmark & \xmark & 42.1  & 56.3 \\

\textbf{InSight-doc} &
\cmark & \cmark & \cmark & \cmark &
\textbf{57.8} & \textbf{65.6} \\
\bottomrule
\end{tabular}
\end{table}

The two benchmarks, DUDE~\citep{van2023document} and MMLongBench-Doc~\citep{ma2024mmlongbench}, contain questions that cannot be answered from the provided documents.
Such questions evaluate whether a model can recognize insufficient evidence rather than hallucinate an unsupported answer.
This setting is particularly relevant for \ourmethod, where \emph{missing visual details may further increase the risk of hallucination}.

Table~\ref{tab:not_answerable_f1} reports F1 scores on the unanswerable subsets.
At $r=0.25$, \ourmethod achieves 69.1 on DUDE and 74.4 on MMLongBench-Doc, improving over Qwen3-VL-8B without zoom by \textbf{24.6} and \textbf{25.9 points}, respectively.
At $r=0.5$, it improves the corresponding scores by \textbf{15.0} and \textbf{19.2 points}.
It also consistently outperforms the zoom-enabled baseline across both datasets and resolutions.
These results indicate that \ourmethod is better able to identify when the document does not contain sufficient evidence, with particularly large gains under low resolution.
Figure~\ref{fig:teaser5} provides a qualitative example on an unanswerable question.

\subsection{Trajectory Quality Analysis}
\label{app:trajectory-quality-full-benchmark}

We analyze the tool-use trajectories of Qwen3-VL-8B with zoom, \ourmethod-8B (SFT), and the final \ourmethod-8B (SFT+RL).
Table~\ref{tab:trajectory-quality-full-benchmark} reports macro-averaged statistics over DUDE, MP-DocVQA, MMLongBench-Doc, and LongDocURL at the low- and medium-resolution settings ($r=0.25$ and $r=0.5$).
Evidence-box coverage is reported only for LongDocURL, the benchmark for which box-level evidence annotations are available.

\begin{table}[t]
\centering
\fontsize{9}{11}\selectfont
\setlength{\tabcolsep}{4.2pt}
\caption{
\textbf{Trajectory-quality comparison on document VQA benchmarks.}
\textit{Crops} denotes the mean number of zoom-in calls.
\textit{Box} denotes the LongDocURL evidence-box hit rate, where a crop must cover at least 50\% of an evidence box.
A trajectory is \textit{redundant} if it contains a pair of crops from the same source with IoU $\geq 0.8$.
A trajectory is \textit{stuck} if it exhausts the tool-call budget and ends with at least two consecutive crops on the same page.
\textit{Area} is the union of all cropped regions normalized by the full-page area.
Box coverage, redundancy, stuck rate, and area are reported as percentages.
}
\label{tab:trajectory-quality-full-benchmark}
\begin{tabular}{lccccccc}
\toprule
\multirow{2}{*}{$r$} & \multicolumn{5}{c}{Overall} & \multicolumn{2}{c}{Unanswerable} \\
\cmidrule(lr){2-6}\cmidrule(lr){7-8}
& Crops & Box & Rdn. & Stuck & Area & Crops & Stuck \\
\hline
\multicolumn{8}{l}{\textcolor{gray}{\fontsize{7}{8}\textit{Qwen3-VL-8B (w/ zoom)}}} \\
\rowcolor{cyan!8}
0.25 & 2.94 & 27.5 & 14.1 & 9.7 & 15.2 & 3.03 & 9.6 \\
0.50 & 1.66 & 41.8 & 6.0 & 4.4 & 11.1 & 1.63 & 4.6 \\
\hline
\multicolumn{8}{l}{\textcolor{gray}{\fontsize{7}{8}\textit{\ourmethod-8B (SFT)}}} \\
\rowcolor{cyan!8}
0.25 & 2.06 & 68.1 & 11.7 & 5.1 & 16.4 & 3.41 & 10.2 \\
0.50 & 1.23 & 70.2 & 4.1 & 1.6 & 12.4 & 2.08 & 3.5 \\
\hline
\multicolumn{8}{l}{\textcolor{gray}{\fontsize{7}{8}\textit{\ourmethod-8B (SFT+RL)}}} \\
\rowcolor{cyan!8}
0.25 & 2.34 & 82.3 & 5.8 & 0.1 & 28.5 & 2.75 & 0.0 \\
0.50 & 1.69 & 77.0 & 2.4 & 0.0 & 22.5 & 1.91 & 0.0 \\
\bottomrule
\end{tabular}
\end{table}

SFT substantially improves evidence localization relative to the zoom-enabled base model.
At $r=0.25$, it reduces the mean number of crops from 2.94 to 2.06 while increasing LongDocURL evidence-box coverage from 27.5\% to 68.1\%.
A similar improvement holds at $r=0.5$, where coverage increases from 41.8\% to 70.2\% with fewer tool calls.
However, SFT alone still produces redundant or stuck trajectories, particularly under aggressive downsampling and on unanswerable questions.

RL further improves both localization quality and trajectory stability.
At $r=0.25$, it raises evidence-box coverage to 82.3\% while using fewer crops than the base model, and reduces redundant and stuck trajectories from 14.1\% and 9.7\% to 5.8\% and 0.1\%, respectively.
At $r=0.5$, it achieves the highest box coverage and the lowest redundancy, while eliminating stuck trajectories almost entirely at both resolutions, including on the unanswerable subsets.
The main trade-off is a larger union crop area: the RL model covers 22.5--28.5\% of a page, compared with 11.1--15.2\% for the base model and 12.4--16.4\% for SFT.
This suggests that RL learns to search more broadly when necessary while avoiding repeated or unproductive tool calls.

\section{Conclusion}

We introduced \textbf{InSight-doc}, which treats visual resolution as an adaptive reasoning-time resource for long-document understanding.
By selectively zooming from low-resolution pages into relevant regions, it improves accuracy while reducing, hallucination, context length and end-to-end inference latency.
Results across document and high-resolution VQA benchmarks demonstrate the effectiveness of coarse-to-fine visual reasoning.

\section*{Limitations}
We only experimented with SFT+RL on Qwen3-VL-8B-Instruct with our proposed framework and dataset.
For a more complete evaluation, additional recent models and models from other providers may be considered.
We did not experiment with any advanced RL methods or reward design.
There may be some room for improvement on this front.


\bibliography{custom}

\clearpage

\appendix

\section{Extended Related Work}
\label{app:related_work}

\subsection{Document Understanding}
Document understanding with MLLMs requires answers to be faithfully grounded in fine-grained visual evidence. However, this remains challenging due to two intertwined bottlenecks. First, MLLMs are prone to \emph{visual hallucination}, producing answers that are insufficiently supported by the document content~\citep{li2023pope}. This issue becomes more pronounced in long-document understanding, where extended contexts can lead to context degradation. Second, it is difficult to balance the trade-off between \emph{fidelity and efficiency}. Preserving fine-grained evidence requires high-resolution image encoding, which substantially increases the number of visual tokens and can exhaust the context budget. Conversely, aggressive downsampling reduces token cost but may discard critical details needed for faithful reasoning~\citep{zheng2026docvstar}.

\paragraph{End-to-end methods.}
The first line of work follows the standard workflow, where all document pages are treated as individual images and directly fed into the model. Most advanced models, such as Qwen-VL series~\cite{bai2025qwen2.5vl, Qwen3-VL} and InternVL series~\cite{zhu2025internvl3, wang2025internvl3.5}, fall into this group. This paradigm preserves fine-grained visual information and generally achieves strong performance. However, it also introduces substantial computational overhead for both training and inference due to the large number of visual tokens. In addition, as many pages may contain task-irrelevant information, scaling this approach to extra-long documents remains challenging. Recently, DeepSeek-OCR~\cite{wei2025deepseek} introduced DeepEncoder and demonstrated strong OCR capability with an aggressive 20$\times$ compression ratio. Nevertheless, it primarily serves as a document parsing model, leaving its effectiveness for document understanding less explored.

\paragraph{Visual retrieval-based methods.}
The second line of work borrows the idea of retrieval-augmented generation (RAG), using textual queries to retrieve a fixed top-\(k\) subset of pages deemed relevant to the question and feeding only these pages to the generator~\citep{yu2024visrag,cho2024m3docrag,faysse2024colpali,chen2024sv,tanaka2025vdocrag,ma2024dse}. For example, VDocRAG~\cite{tanaka2025vdocrag} first employs VDocRetriever to compute the embedding similarity between the \texttt{[EOS]} representations of the text query and document images. The retrieved top-\(k\) pages are then fed into VDocGenerator for question answering. While this paradigm reduces the input to a manageable size, retrieval and reasoning remain loosely coupled: the generator has limited ability to recover from initial retrieval errors, performance can be sensitive to the choice of \(k\), and multi-hop questions may be affected when relevant evidence falls outside the retrieved subset.
Multi-turn retrieval~\citep{wang2026vrag} may mitigate these issues but the fundamental limitations remain.


\paragraph{Coarse-to-fine methods.}
More recently, a third line of work adopts a coarse-to-fine strategy to balance resolution and token cost: the document is first scanned at low resolution to identify candidate pages, and the selected pages are then encoded at high resolution for detailed inspection~\citep{xu2025cogdoc,zheng2026docvstar}.
CogDoc~\citep{xu2025cogdoc} first localizes relevant pages from a low-resolution document, then makes a single transition to high-resolution reasoning over that selected subset.
Doc-V$^\star$~\citep{zheng2026docvstar} is the closest to InSight-doc in spirit. 
Doc-V$^\star$ formulates multi-page Document VQA as a sequential evidence aggregation process, where an OCR-free MLLM agent starts from a global thumbnail overview and iteratively retrieves or fetches target pages for grounded reasoning.
However, Doc-V$^\star$ primarily (94.0–99.8\%) relies on an external retrieval model, Colqwen2.5~\citep{faysse2024colpali}, whereas InSight-doc is fully end-to-end.
This introduces indexing/retrieval errors (at the retrieval side) which they may not be able to recover from.
The reliance on the external retrievers also introduces indexing/retrieval overhead which slows down the inference process.
Finally, both CogDoc and Doc-V$^\star$ still localize evidence primarily at the page level, requiring one or more entire pages to be encoded at high resolution. As a result, they may allocate many visual tokens to irrelevant page content and remain less precise in isolating the sub-page regions that contain the answer, especially when multiple pages need to be revisited.

\subsection{Visual Search}
\label{app:related_work_vs}
Early visual search methods rely on external detectors or scripted workflows to localize regions and trigger tool use via instruction tuning, dynamically acquiring higher-resolution evidence but typically in a single, rigid round of search~\citep{vstar,shen2024zoomeye,li2025dyfo}. The ``think with images'' paradigm popularized by OpenAI o3~\citep{openai2025o3} internalizes zoom and crop as intrinsic operations, allowing the model to seek and integrate \emph{multi-scale} visual evidence within an image--text interleaved reasoning trace. Building on this idea, recent works train MLLMs to ``think with images'' through reinforcement learning (DeepEyes~\citep{zheng2025deepeyes}), synthetic warm-starts (Pixel-Reasoner~\citep{su2025pixel}), and multi-turn RL with over-turn masking (Mini-o3~\citep{lai2025mini}), enabling the model to iteratively decide \emph{when} to zoom, \emph{where} to look, and \emph{how} to fuse evidence across scales. More recently, InSight-o3~\citep{li2026insighto} decouples visual reasoning from visual search by introducing a dedicated search agent that localizes fuzzy, relational, or conceptual regions on arbitrary images, broadening \emph{multi-scale} evidence acquisition beyond discrete object references on natural photographs.

\subsection{Visual Search on Videos}
\label{app:related_work_video}
A long document can be viewed as a sequence of document pages, analogous to consecutive frames in a video. A parallel research thread studies visual search over long videos, where models locate sparse, query-relevant frames or clips from a long temporal stream rather than search across high-resolution document pages. VideoDeepResearch~\citep{yu2025videodeepresearch} uses a text-only large reasoning model with a modular multimodal toolkit to plan which video segments to retrieve and inspect. LongVideo-R1~\citep{qiu2026longvideo} organizes videos hierarchically and trains an MLLM agent to traverse summaries and iteratively focus on informative clips. ProVCA~\citep{yin2026provca} progressively condenses videos at multiple temporal granularities to extract query-relevant frames under tight compute budgets.
While these methods share our goal of selectively gathering sparse evidence instead of encoding the entire visual input, they operate primarily along the temporal axis with relatively low per-frame resolution. Multi-page document understanding, in contrast, requires localizing sub-page regions within individually high-resolution pages, with evidence often scattered across distant and non-adjacent pages.

\begin{figure*}[t]
    \centering
     \includegraphics[width=1.0\linewidth]{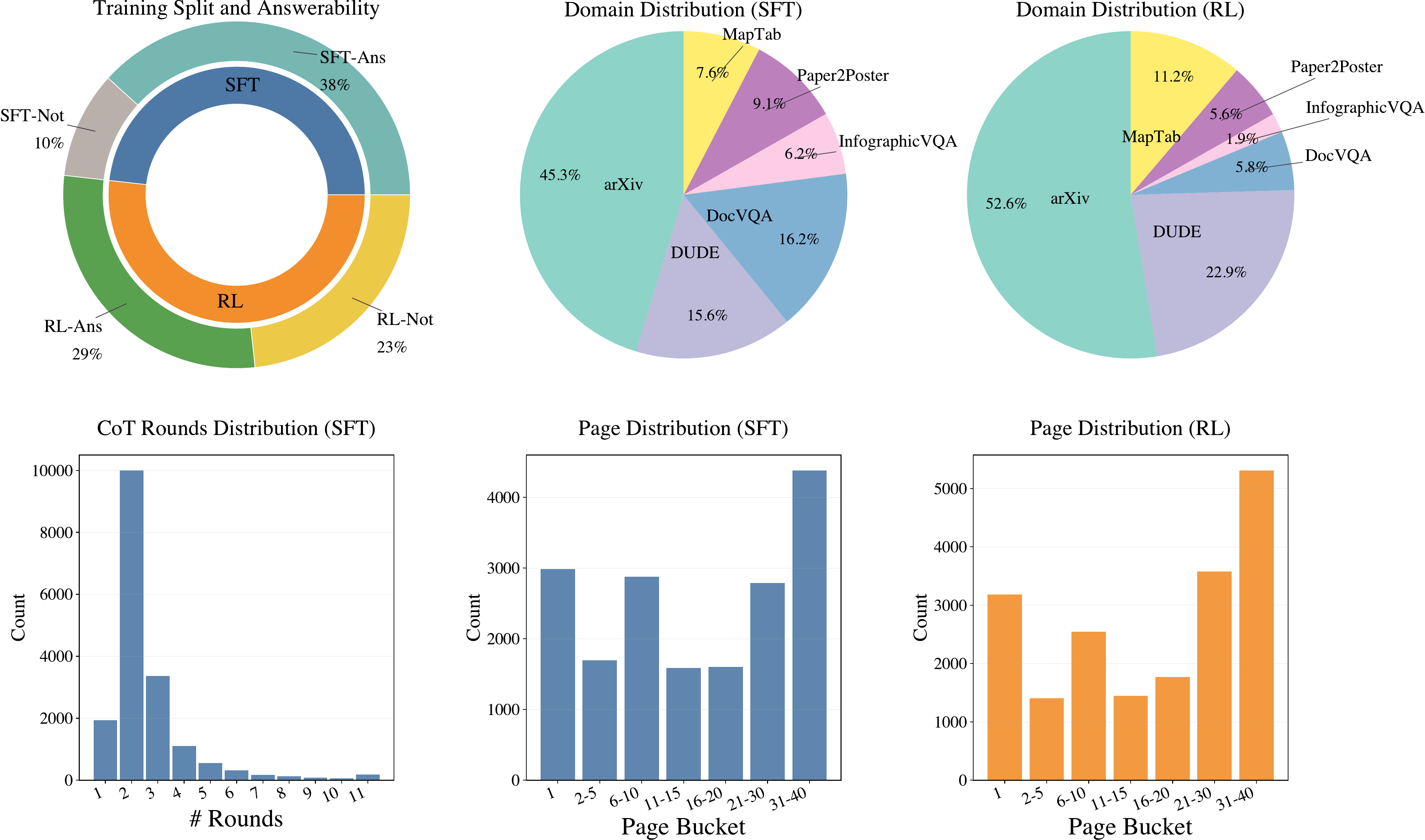}
      \caption{\textbf{Statistics of our training data.}}
    \label{fig:data_dist}
\end{figure*}

\section{Proofs}

\subsection{Proof of Proposition~\ref{prop:relative_sequence_length}}
\label{app:proof-relative-seq-len}

\begin{proof}
By definition, $P_r=x(r)P_0$ and $R_r=y(r)R_0$. Therefore,
\begin{equation}
S_r
=
P_r+R_r
=
x(r)P_0+y(r)R_0.
\end{equation}
Using $\kappa=P_0/R_0$, we have $P_0=\kappa R_0$, and hence
\begin{equation}
\frac{S_r}{S_0}
=
\frac{x(r)P_0+y(r)R_0}{P_0+R_0}
=
\frac{\kappa x(r)+y(r)}{\kappa+1}.
\end{equation}
Because $\kappa>0$ and $x(r),y(r)\geq0$, we obtain
\begin{equation}
\frac{\kappa x(r)+y(r)}{\kappa+1}
\leq
\frac{\kappa x(r)+y(r)}{\kappa}
=
x(r)+\kappa^{-1}y(r).
\end{equation}
Therefore,
\begin{equation}
\frac{S_r}{S_0}
\leq
x(r)+\kappa^{-1}y(r),
\end{equation}
which proves the result.
\end{proof}

\subsection{Proof of Proposition~\ref{prop:relative_latency}}
\label{app:proof-relative-latency}

\begin{proof}
Consider an execution at resize ratio $r$ with $m=n(r)$ tool calls. Let
$U_0$ denote the initial input, let $U_j$ for $1\leq j\leq m$ denote the
input returned by the $j$-th tool call, and let $G_j$ for $0\leq j<m$
denote the intermediate generated responses. Let $G_m$ denote the final
response. Their total token counts satisfy
\begin{align*}
P_r&=\sum_{j=0}^{m}U_j=x(r)P_0,\\
R_r&=\sum_{j=0}^{m}G_j=y(r)R_0.
\end{align*}

First consider a hypothetical single-turn execution in which all $P_r$ input
tokens precede all $R_r$ generated tokens. Under
Eq.~\eqref{eq:single_turn_cost}, its latency is
\[
T(P_r,R_r)
=
\alpha P_r^2+\beta R_r(2P_r+R_r).
\]

The actual execution is interleaved as
\[
U_0,G_0,U_1,G_1,\ldots,U_m,G_m.
\]
With prefix caching, previously processed tokens are not recomputed. The only
difference from the hypothetical ordering concerns pairs consisting of an
earlier generated chunk $G_i$ and a later tool-return chunk $U_j$, where
$i<j$. In the hypothetical single-turn ordering, such pairs contribute
$2\beta G_iU_j$ through the prompt--generation term. In the actual ordering,
$U_j$ is processed during prefill and the same pairs contribute
$2\alpha G_iU_j$. Therefore,
\begin{equation}
\label{eq:multiturn_cost_decomposition}
T_r
=
T(P_r,R_r)
-
2(\beta-\alpha)
\sum_{0\leq i<j\leq m}G_iU_j.
\end{equation}

Since $\beta\geq\alpha$ and all token counts are nonnegative,
Eq.~\eqref{eq:multiturn_cost_decomposition} implies
\[
T_r\leq T(P_r,R_r).
\]
Substituting $P_r=x(r)P_0$ and $R_r=y(r)R_0$ gives
\[
T_r
\leq
\alpha x(r)^2P_0^2
+
\beta y(r)R_0
\left(2x(r)P_0+y(r)R_0\right).
\]

The no-resize baseline has latency
\[
T_0
=
\alpha P_0^2+\beta R_0(2P_0+R_0).
\]
Dividing the preceding inequality by $T_0$, and then dividing its numerator
and denominator by $\beta R_0^2$, yields
\[
\frac{T_r}{T_0}
\leq
\frac{
\gamma\kappa^2x(r)^2
+
2\kappa x(r)y(r)
+
y(r)^2
}{
\gamma\kappa^2+2\kappa+1
},
\]
where $\gamma=\alpha/\beta$ and $\kappa=P_0/R_0$. Using
\[
(w_{\mathrm{p}},w_{\mathrm{c}},w_{\mathrm{g}})
=
\frac{(\gamma\kappa^2,\,2\kappa,\,1)}
{\gamma\kappa^2+2\kappa+1}
\]
gives
\[
\frac{T_r}{T_0}
\leq
w_{\mathrm{p}}x(r)^2
+
w_{\mathrm{c}}x(r)y(r)
+
w_{\mathrm{g}}y(r)^2,
\]
which proves the proposition.
\end{proof}


\section{Data Construction Details}
\label{app:data_details}

This appendix provides the additional details of our data construction pipeline that were omitted from Section~\ref{sec:data}, including per-source sampling strategies, multi-page and multi-hop construction templates, and dataset statistics.

\subsection{Per-Source Sampling Strategies}

\begin{itemize}
    \item \textbf{arXiv}~(Kaggle snapshot): we sample documents following the overall domain distribution of the repository, prioritising longer papers that are more likely to contain dense figures and tables.
    \item \textbf{DUDE, DocVQA, InfographicVQA}: we draw exclusively from the \emph{training} splits and retain only a subset of each to keep the pipeline cost-bounded.
    \item \textbf{Paper2Poster}: we use the author-designed posters and the proposed QA pairs, filtered to keep only those QAs answerable from the posters.
    \item \textbf{MapTab}: we use re-generated QAs grounded on the metro/travel maps spanning 160 cities.
\end{itemize}

\subsection{Multi-page Construction}

QAs from the non-arXiv sources are originally grounded on a single page, and some documents are short. We construct longer-context variants in two ways:

\paragraph{DocVQA expansion.} For a sampled DocVQA question whose evidence page is \(p\), we retrieve the full source document containing \(p\) and rewrite the question (with Gemini~3.1 Flash-Lite, conditioned on the original QA and the surrounding pages) so that it remains answerable but is no longer trivially scoped to a single page. The answer is preserved.

\paragraph{Random multi-document merging.} For the remaining sources, we sample a target page count and concatenate the question-bearing document with several additional documents drawn at random from the same source. The order of documents in the merged sequence is shuffled, while the internal page order within each document is preserved. The original answer remains correct because evidence pages are kept verbatim, but the model must locate them within a much longer context. To avoid invalidating the original questions, we also use Gemini~3.1 Flash-Lite to rewrite the question so that the evidence page remains identifiable from the question itself.

\subsection{QA Construction}
\label{app:qa-construction}

\subsubsection{arXiv QA construction}
\label{app:arxiv-qa-construction}

The arXiv questions are constructed by an evidence-centric pipeline
over scientific papers, rather than sampled directly from a generic document
VQA pool. The pipeline starts by enriching figure and table elements from each
paper: for each document, up to 20 visual elements are selected, and a strong
MLLM produces local context, captions, and explanations using the paper text
and page images. Documents whose enrichment remains empty or failed after
retries are dropped. These enriched visual elements form the evidence inventory
for both single-hop and multi-hop QA construction.

\paragraph{Single-hop visual-element QA.}
For single-hop QA, the pipeline generates questions from individual enriched
visual elements using rendered paper pages. Each candidate question
is anchored to a specific visual element and its page, and the generation
prompt includes the element image, its enriched description, and surrounding
paper context. The generated answer is then checked by a separate evaluation
pass against the rendered evidence page or element. Candidates are retained
only when the answer is supported by the designated visual evidence and the
question is sufficiently grounded in the scientific content. This stage
therefore produces localized visual QA over equations, tables, plots,
diagrams, and nearby scientific text, while retaining evidence pages and, when
available, evidence boxes for later auditing.

\paragraph{Multi-hop visual QA.}
For multi-hop QA, we sample groups of \(k\in\{2,3,4\}\) visuals from the same
paper, such as a table plus one or more figures, and prompt the MLLM with the
visuals and their enriched descriptions to generate questions whose answer can
only be derived by jointly reasoning over all sampled visuals, such as
comparison or aggregation questions.

The multi-hop stage uses rendered pages for candidate generation and applies additional dependency checks: a
minimum of two distinct evidence groups, full-group visual-consistency trials,
and leave-one-out evidence checks. Full-group visual-consistency trials ask the
evaluator to answer using all proposed evidence groups and reject candidates
whose answer is not stable when the complete evidence is visible.
Leave-one-out checks remove one evidence group at a time and require the answer
to fail or become uncertain, which filters out questions that are actually
answerable from only a subset of the visuals.

After the dependency checks,
there are two auxiliary passes. PDF follow-up is an audit
pass that asks whether the candidate
answer is correct when the full PDF context is visible, and whether it remains
answerable when the proposed evidence union is masked or excluded. This helps
detect evidence leakage, i.e., candidates whose answer can be recovered from
other pages or from non-designated visual content.
Reference-answer refinement,
in contrast, directly affects the final QA item: after a candidate is selected,
the answer is regenerated or normalized from higher-resolution renderings of the selected evidence pages. This reduces OCR mistakes,
low-resolution visual ambiguity, and formatting artifacts in the final gold
answer.

Together, these stages reject candidates that are answerable from only
one visual, have unstable answers, ambiguous evidence assignments, evidence
leakage, or noisy target answers.

\subsubsection{Non-arXiv QA construction}
\label{app:non-arxiv-qa-construction}

For non-arXiv sources, questions are taken from existing document
VQA-style datasets and task-specific generation pipelines, including
MP-DocVQA, DUDE, poster QA, infographic QA, and map QA. These sources are more
heterogeneous: some rows require localized visual reasoning, while others are
simple OCR or local text lookup.

\paragraph{Compound QA construction.}
For both arXiv and non-arXiv rows, we can also synthesize multi-hop QAs by
combining several existing single-hop QAs anchored to the same, possibly
merged, document into a compound item:
\begin{quote}
\textbf{Question:} (a)~question 1; (b)~question 2; \ldots\\
\textbf{Answer:}   (a)~answer 1;   (b)~answer 2;   \ldots
\end{quote}
This template guarantees correctness while still requiring the model to attend to multiple, spatially separated parts of the document.

\subsection{SFT Data Construction}
\label{app:sft-data-construction}

\subsubsection{Answerable questions}
\label{app:sft-answerable-construction}

\paragraph{Source data.}
Answerable SFT examples are constructed from a mixture of non-arXiv document
VQA sources and arXiv-derived question-answering sources. The non-arXiv sources
include MP-DocVQA, DUDE, poster QA, infographic QA, and map QA. The
arXiv-derived sources include visually grounded equation/table/text questions
and multi-evidence questions from scientific papers. For answerable examples,
the source data may contain evidence pages and, for a subset, evidence boxes.
These annotations are used only for data auditing and trajectory-quality
analysis; they are not exposed to the policy during training.



\paragraph{Difficulty filtering.}
We apply the data filtering stages defined in the main paper after the
source-specific construction above and before generating tool-use trajectories.
For heterogeneous non-arXiv sources, the first stage removes answerable rows that a low-resolution Qwen3-VL-8B model already solves, since these sources contain many simple OCR or local text-lookup questions.
We do not apply this 8B gate to the answerable arXiv rows because they are not sampled from the same broad, weakly controlled pool.
Instead, they are generated by an evidence-centric scientific-paper
pipeline with several explicit dependency checks as discussed in Appendix~\ref{app:arxiv-qa-construction}.
These checks directly target the main failure mode
that the 8B gate removes for non-arXiv data: questions that can be solved
without finding the intended evidence. Skipping the 8B gate therefore avoids a
redundant source-mismatched filter, while all arXiv rows are still subjected to
the stronger Qwen3-VL-32B filtering stage before trajectory generation. For
answerable rows, this final filter keeps examples that require stronger
document reasoning or tool use.
Table~\ref{tab:sft-filter-stages} summarizes the number of answerable rows that enter each filtering branch, grouped by source family and DPI.

\begin{table*}[t]
\centering
\caption{\textbf{Filtering stages for answerable rows before SFT trajectory generation.} Each cell reports row count followed
by retention relative to the corresponding source-family/DPI source pool. The
DPI columns correspond to resize ratios \(r=0.25,0.35,0.5\) from 200-DPI page
renders. The arXiv rows bypass prior-only
filtering and are therefore carried forward unchanged in the second group; all
sources are then included in zoom-free filtering.
Unanswerable rows are excluded from this table and described separately.}
\label{tab:sft-filter-stages}
\small
\setlength{\tabcolsep}{5pt}
\begin{tabular}{@{}lrrrr@{}}
\toprule
Filtering stage & Total & 50 DPI ($r=0.25$) & 70 DPI ($r=0.35$) & 100 DPI ($r=0.5$) \\
\midrule
\multicolumn{5}{@{}l}{\textit{All sources}} \\
Source pool before filtering & 50{,}903 (100.0\%) & 25{,}344 (100.0\%) & 14{,}839 (100.0\%) & 10{,}720 (100.0\%) \\
After prior-only filtering & 44{,}889 (88.2\%) & 22{,}382 (88.3\%) & 13{,}153 (88.6\%) & 9{,}354 (87.3\%) \\
After zoom-free filtering & 26{,}943 (52.9\%) & 15{,}118 (59.7\%) & 7{,}370 (49.7\%) & 4{,}455 (41.6\%) \\
\addlinespace[2pt]
\multicolumn{5}{@{}l}{\textit{arXiv}} \\
Source pool before filtering & 16{,}949 (100.0\%) & 8{,}367 (100.0\%) & 5{,}582 (100.0\%) & 3{,}000 (100.0\%) \\
After prior-only filtering & 16{,}949 (100.0\%) & 8{,}367 (100.0\%) & 5{,}582 (100.0\%) & 3{,}000 (100.0\%) \\
After zoom-free filtering & 11{,}885 (70.1\%) & 6{,}541 (78.2\%) & 3{,}654 (65.5\%) & 1{,}690 (56.3\%) \\
\addlinespace[2pt]
\multicolumn{5}{@{}l}{\textit{non-arXiv}} \\
Source pool before filtering & 33{,}954 (100.0\%) & 16{,}977 (100.0\%) & 9{,}257 (100.0\%) & 7{,}720 (100.0\%) \\
After prior-only filtering & 27{,}940 (82.3\%) & 14{,}015 (82.6\%) & 7{,}571 (81.8\%) & 6{,}354 (82.3\%) \\
After zoom-free filtering & 15{,}058 (44.3\%) & 8{,}577 (50.5\%) & 3{,}716 (40.1\%) & 2{,}765 (35.8\%) \\
\bottomrule
\end{tabular}
\end{table*}

\paragraph{Trajectory generation with InSight-o3.}
Rows that pass the filters are given to InSight-o3~\citep{li2026insighto}, which generates two-agent multi-turn trajectories.
The original InSight-o3 does not consider multi-image input, so we introduce an additional parameter \texttt{img\_idx} to the vSearcher tool so the vReasoner can tell the vSearcher which image to look at.
The trajectories are then pieced together to form unified, coherent single-agent trajectories.
Specifically, we replace every vSearcher call in a vReasoner trajectory with a zoom-in tool call where the \texttt{img\_idx} parameter comes from the same parameter of the vSearcher call, the \texttt{label} parameter comes from the \texttt{region\_description} parameter of the vSearcher call, and the \texttt{bbox} parameter comes from the region bounding box (after rescaling) returned by the vSearcher call.
A valid trajectory contains the original user question, one or more assistant tool calls, tool-returned image crops, and a final assistant answer that matches the ground-truth answer.
Questions whose trajectories are invalid are reserved for RL.

\paragraph{Degenerate-trajectory removal.}
After generating the trajectories, we remove ones that are not useful for supervised training.
This includes malformed conversations, invalid or missing tool-call arguments, missing final answers, image references that cannot be resolved, and other cases where the converted example would teach an inconsistent interaction pattern.

\paragraph{Final-response normalization.}
For the retained trajectories, we rewrite only the final assistant response
with GPT-5-nano. The tool calls, crop sequence, and tool observations are left
unchanged. We do this because the supervision signal should teach the agent
where and when to zoom, while the raw InSight-o3 final response is in a \texttt{think}+\texttt{answer} style. This is useful
for trajectory generation, but it is not the desired final-answer format for
Qwen3-VL-Instruct-style SFT targets. If kept unchanged, the SFT loss would
optimize for extra reasoning text, detached final-answer lines, self-corrections,
and formatting artifacts. The rewrite converts this artifact into a natural
plain-text answer that preserves the original answer content, folds brief
evidence into normal prose when needed, and removes unnecessary explanation and
formatting variation.

\begin{table*}[t]
\centering
\small
\setlength{\tabcolsep}{3pt}
\caption{\textbf{Data-flow summary for filtering, add-ons, and the final SFT/RL split.}
Each numeric cell is an exact row count; add-on rows with a leading ``+''
are incremental counts, while other rows are cumulative totals. Within each
source-family group, ``All'' is the sum of answerable (Ans.) and unanswerable
(Unans.) rows.}
\label{tab:data-flow-summary}
\begin{tabular}{@{}lrrrrrrrrr@{}}
\toprule
& \multicolumn{3}{c}{\textbf{All sources}} & \multicolumn{3}{c}{\textbf{arXiv}} & \multicolumn{3}{c}{\textbf{non-arXiv}} \\
\cmidrule(lr){2-4}\cmidrule(lr){5-7}\cmidrule(l){8-10}
\textbf{Stage} & All & Ans. & Unans. & All & Ans. & Unans. & All & Ans. & Unans. \\
\midrule
\multicolumn{10}{@{}l}{\textit{Filtering before trajectory generation}} \\
Source pool before filtering & 62{,}318 & 50{,}903 & 11{,}415 & 22{,}552 & 16{,}949 & 5{,}603 & 39{,}766 & 33{,}954 & 5{,}812 \\
After prior-only filtering & 56{,}304 & 44{,}889 & 11{,}415 & 22{,}552 & 16{,}949 & 5{,}603 & 33{,}752 & 27{,}940 & 5{,}812 \\
After zoom-free filtering & 33{,}502 & 26{,}943 & 6{,}559 & 13{,}436 & 11{,}885 & 1{,}551 & 20{,}066 & 15{,}058 & 5{,}008 \\
\addlinespace[2pt]
\multicolumn{10}{@{}l}{\textit{SFT construction}} \\
Correct InSight-o3 trajectories & 14{,}717 & 14{,}216 & 501 & 6{,}350 & 6{,}306 & 44 & 8{,}367 & 7{,}910 & 457 \\
Synthetic unanswerable add-on & +3{,}196 & +0 & +3{,}196 & +1{,}765 & +0 & +1{,}765 & +1{,}431 & +0 & +1{,}431 \\
Final SFT rows & 17{,}913 & 14{,}216 & 3{,}697 & 8{,}115 & 6{,}306 & 1{,}809 & 9{,}798 & 7{,}910 & 1{,}888 \\
\addlinespace[2pt]
\multicolumn{10}{@{}l}{\textit{RL construction}} \\
RL candidate pool & 18{,}785 & 12{,}727 & 6{,}058 & 7{,}086 & 5{,}579 & 1{,}507 & 11{,}699 & 7{,}148 & 4{,}551 \\
After source selection, cleanup, and 24k cap & 11{,}719 & 7{,}403 & 4{,}316 & 6{,}125 & 4{,}618 & 1{,}507 & 5{,}594 & 2{,}785 & 2{,}809 \\
Synthetic unanswerable add-on & +4{,}341 & +0 & +4{,}341 & +2{,}415 & +0 & +2{,}415 & +1{,}926 & +0 & +1{,}926 \\
Multiple-choice add-on & +2{,}176 & +2{,}176 & +0 & +581 & +581 & +0 & +1{,}595 & +1{,}595 & +0 \\
Structured-document add-on & +1{,}000 & +1{,}000 & +0 & +1{,}000 & +1{,}000 & +0 & +0 & +0 & +0 \\
Final RL rows (w/o reweighting) & 19{,}236 & 10{,}579 & 8{,}657 & 10{,}121 & 6{,}199 & 3{,}922 & 9{,}115 & 4{,}380 & 4{,}735 \\
\bottomrule
\end{tabular}
\end{table*}

\subsubsection{Unanswerable questions}
\label{app:sft-unanswerable-construction}

\paragraph{Source unanswerable rows.}
The first source of unanswerable SFT data is naturally occurring unanswerable
questions from non-arXiv data. DUDE provides explicit unanswerable labels. The
poster source does not provide answerability labels directly: some questions
were generated from the underlying paper rather than from the poster itself, so
they are not necessarily answerable from the poster image. We mine these poster
unanswerables by asking a strong MLLM, GPT-5, to answer each question using only
the high-resolution poster; questions whose poster-only answer does not match
the original answer are treated as unanswerable from the poster. These source
unanswerable rows bypass the Qwen3-VL-8B super-easy gate, which is only used to
remove answerable rows solved by the low-resolution 8B model, but they are still
included in the later Qwen3-VL-32B filtering stage.

\paragraph{Filtering unanswerable rows.}
For unanswerable questions, the filtering objective differs from the answerable
case: a row is useful when a strong no-tool model still fails to abstain
correctly, since such examples teach the agent not to hallucinate unsupported
answers. The non-arXiv source pool contains 5{,}812 real unanswerable rows. All
of them bypass the 8B super-easy gate, and the Qwen3-VL-32B no-tool filtering
stage retains 1{,}614 of them as trajectory-generation candidates.

\paragraph{Synthetic mutation-based unanswerable add-on.}
We also add a separate verified-unanswerable branch before trajectory
generation. This add-on has two goals. First, it creates hard negatives that
look very similar to normal answerable questions: each question is only a small
natural perturbation of an answerable seed. Second, it expands unanswerable
coverage beyond the DUDE and poster sources. The seed questions are sampled from
four non-arXiv training partitions and from three arXiv answerable pools: the
base arXiv pool, an additional arXiv pool, and an arXiv spanning-question pool.
The seed sampler uses a roughly balanced mixture between non-arXiv and arXiv
seeds.

For each seed, we first show GPT-5-nano the original question, answer, question
type, and the pages relevant to the seed question, and ask it to make a small
natural mutation that keeps the question document-related but removes document
support~\citep{gautam2023lightweight}. Typical mutations swap an entity, number, date, value, attribute, or
comparison target; the prompt explicitly rejects nonsensical questions,
missing-page questions, hidden-label questions, and dataset artifacts. We then
verify each candidate with Gemini 3.1 Flash Lite Preview using a broader page
window from the same document. The verifier separates genuinely unanswerable
cases from answerable, externally answerable, malformed, generic-answer, or
ambiguous cases. We retain only candidates labeled as missing evidence or
document mismatch, and use the verifier's concise insufficient-evidence answer
as the target answer. Because these rows are explicitly constructed and verified
to be minimal document-grounded negatives, we do not pass them through the
Qwen3-VL-32B difficulty gate; applying that gate would mainly remove clean
abstention examples that the strong model already handles, rather than improve
the quality of the add-on. Accepted rows are rendered at resize ratios
$r \in \{0.25, 0.35, 0.5\}$. The final add-on contains 1{,}582 rows at
$r=0.25$, 985 rows at $r=0.35$, and 629 rows at $r=0.5$. These examples do not
have evidence pages or evidence boxes by definition.

\paragraph{Trajectory generation and postprocessing.}
The retained source unanswerable rows and synthetic mutation-based add-on rows
are passed through the same InSight-o3 trajectory generation, SFT conversion,
degenerate-trajectory removal, and final-response rewriting steps as the
answerable rows. The only semantic difference is the target behavior: the final
response should clearly state that the requested information is missing or
unsupported by the provided document, rather than hallucinating a concrete
answer.
Table~\ref{tab:data-flow-summary} summarizes how the data change across the
filtering stages, add-ons, and final SFT/RL split.

\subsection{RL Data Construction}
\label{app:rl-data-construction}

\paragraph{Source rows and selection.}
The RL data are built from questions that remain useful after the filtering and
SFT-trajectory construction stages. Intuitively, successful InSight-o3
trajectories provide supervised targets for SFT, while the remaining hard
questions and negative examples form the starting pool for RL. We then apply
the same row-level cleanup used elsewhere in the data pipeline: source
selection, duplicate removal, SFT-overlap removal, and prompt-length capping.
This yields the source-derived RL rows shown in
Table~\ref{tab:data-flow-summary}. We additionally include a verified
synthetic unanswerable add-on.
The unanswerable pool therefore combines naturally
unanswerable DUDE/poster rows and mutation-verified synthetic negatives.

\paragraph{Prompt-length capping and resize-ratio recovery.}
RL prompts are capped at 24k estimated prompt tokens, including both text
tokens and image tokens. When an $r=0.5$ row exceeds the cap, we regenerate the
same example at $r=0.35$ and preserve the original resize ratio in the row
metadata. Rows that still exceed the cap after this recovery step are dropped.
This produces the post-cap base RL set before targeted add-ons.

\paragraph{Additional targeted rows.}
After prompt capping, we append two answerable add-ons. The multiple-choice
add-on targets false-negative behavior: the model is shown answerable questions
with a ``not enough information'' option, but the correct choice is one of the
evidence-supported answer options. This discourages erroneous abstention when
sufficient visual evidence is present. The structured-document add-on increases
coverage of arXiv questions that require interpreting document structure. The
final RL parquet contains 19{,}236 rows before weighted sampling.

\subsection{Weighted Refill Sampling for RL}
\label{app:weighted-sampling}

The final RL parquet is not physically balanced. Its raw row
distribution contains substantially more unanswerable rows than we want to
sample during RL training. We therefore balance the training stream with a
weighted refill sampler. For each draw, the sampler chooses a data source from
a YAML weight table, samples a row without replacement from that source-specific
pool, and reshuffles/refills only that source when its pool is exhausted. This
keeps the per-batch mixture close to the target probabilities while still
avoiding repeated rows within a source until that source has been exhausted.
The total sampling mass is 86\% answerable and 14\% unanswerable. The category
weights are summarized in Table~\ref{tab:rl-sampling-targets}.

\begin{table}[t]
\centering
\small
\setlength{\tabcolsep}{5pt}
\caption{RL sampling targets used by the weighted refill sampler. The weights
sum to 1.0.}
\label{tab:rl-sampling-targets}
\begin{tabular}{lr}
\toprule
Sampling target & Weight mass \\
\midrule
Answerable rows & 86.0\% \\
Unanswerable rows & 14.0\% \\
\midrule
arXiv visually grounded QA & 16.04\% \\
arXiv multi-evidence QA & 15.39\% \\
arXiv structural rewrites & 5.00\% \\
DocVQA & 9.97\% \\
DUDE & 22.22\% \\
Infographic QA & 3.99\% \\
Map metro & 15.06\% \\
Map travel & 4.75\% \\
Poster QA & 7.59\% \\
\bottomrule
\end{tabular}
\end{table}

\section{Dataset Statistics and Quality Analysis}
\label{app:dataset-statistics-quality}

\subsection{Basic Statistics}
\label{app:dataset-statistics}

Figure~\ref{fig:data_dist} shows the basic statistics of our SFT and RL datasets.
The final SFT data contain 17{,}913 trajectories, including 14{,}216 answerable
rows and 3{,}697 unanswerable rows. The final RL parquet contains 19{,}236 rows before weighted sampling. Its raw row distribution is 10{,}579 answerable rows and 8{,}657 unanswerable rows; during RL training, the sampler changes the
effective answerable/unanswerable ratio to 86\%/14\%.

The resize ratio $r$ controls the image resolution used when constructing the
prompt: each page image is resized before tokenization, so the prompt-length
budget includes both text tokens and image tokens. Table~\ref{tab:resize-ratio}
shows the resulting distribution. RL is more conservative than SFT because the
24k prompt cap causes some $r=0.5$ examples to be recovered at $r=0.35$.

\begin{table}[t]
\centering
\small
\caption{Resize-ratio distribution of the SFT+RL data.}
\label{tab:resize-ratio}
\setlength{\tabcolsep}{4pt}
\begin{tabular}{lrrr}
\toprule
Dataset & $r=0.25$ & $r=0.35$ & $r=0.5$ \\
\midrule
SFT & 10{,}051 (56.1\%) & 4{,}913 (27.4\%) & 2{,}949 (16.5\%) \\
RL & 13{,}281 (69.0\%) & 4{,}394 (22.8\%) & 1{,}561 (8.1\%) \\
\bottomrule
\end{tabular}
\end{table}

\subsection{SFT Trajectory Quality}
\label{app:sft-trajectory-quality}

\paragraph{Evaluation protocol.}
For trajectory \(i\), let \(C_i\) be its sequence of zoom-in crops, \(P_i\)
the annotated evidence pages, and \(B_i\) the annotated evidence boxes. For a
crop \(c\) and evidence box \(b\) on the same original page, define evidence
coverage as
\[
\mathrm{cov}(c,b)=\frac{\mathrm{area}(c\cap b)}{\mathrm{area}(b)},
\]
and define crop/evidence IoU in the standard way. We use a region-hit threshold
\(\tau=0.5\). Unless otherwise stated, evidence metrics are computed only on
answerable rows with recoverable evidence annotations; unanswerable rows and
rows without evidence annotations are excluded.

We report the following metrics:
\begin{itemize}
\item \textbf{Coverage.} Evidence-page hit rate is the fraction of page-evidence
trajectories with \(|C_i|>0\) for which at least one crop lands on a page in
\(P_i\). Evidence-region hit rate is the fraction of box-evidence trajectories
with \(|C_i|>0\) for which \(\max_{c\in C_i,b\in B_i}\mathrm{cov}(c,b)\geq\tau\).
Mean max evidence coverage averages \(\max_{c,b}\mathrm{cov}(c,b)\) over cropped
box-evidence trajectories.
\item \textbf{Localization precision.} Mean max crop/evidence IoU averages
\(\max_{c,b}\mathrm{IoU}(c,b)\) over cropped box-evidence trajectories. Crop
region-hit rate is a crop-level statistic: the number of crops in box-evidence
trajectories satisfying \(\max_b\mathrm{cov}(c,b)\geq\tau\), divided by the
total number of crops in those trajectories. Crops per evidence-region-hit crop
is the inverse-style cost statistic: total crops in box-evidence trajectories
divided by the number of region-hit crops. Crop area fraction is computed over
the original full pages that were cropped at least once, with overlapping crop area counted only once.
\item \textbf{Efficiency and redundancy.} Same-source
overlap rate is the fraction of trajectories with at least
one crop pair from the same source image whose IoU is at least 0.8. Stuck
rate is the fraction of trajectories that exhaust the tool-call limit and whose
final consecutive crop run contains at least two crops on the same page. Stop exactly at first region hit is computed over trajectories that do
hit an evidence region and checks whether the final crop is the first crop that
satisfies \(\mathrm{cov}\geq\tau\).
\end{itemize}

\paragraph{Summary.}
The SFT trajectories have high evidence coverage: among answerable examples
with evidence annotations and at least one crop, 95.33\% hit an evidence page
and 85.00\% cover at least half of an evidence box. They are also reasonably
precise spatially: the average crop union covers 14.50\% of full-page area,
65.34\% of crops in box-evidence trajectories are evidence-region hits, and the
mean best crop/evidence IoU is 52.37\%. Finally, the trajectories are usually
efficient rather than repetitive: the mean trajectory uses 1.61 crops, most
rows use one crop, near-duplicate same-source crops occur in 3.16\% of trajectories, and only 0.99\% of such trajectories are
classified as stuck.

\subsection{Qualitative Examples}

See an example of our training data in Figure \ref{fig:traing_data}.
More examples can be found on Hugging Face.

\section{Experiment Settings}

\subsection{Evaluation Datasets}
\label{app:eval-data}
We conduct a comprehensive evaluation on six datasets that span both standard and challenging long-document scenarios. Specifically, we evaluate on the validation sets of \textbf{DUDE} \citep{van2023document} and \textbf{MP-DocVQA (MP-DVQA)} \citep{tito2023hierarchical}, two widely adopted multi-page document VQA benchmarks. To further assess performance under extreme document lengths, we include two challenging long-document benchmarks: \textbf{MMLongBench-Doc (MMLong.)}  \citep{ma2024mmlongbench}, which contains documents of up to 468 pages, and \textbf{LongDocURL (LongDoc.)} \citep{deng2025longdocurl}, with 50--150 pages for all documents.
Unless stated otherwise, evaluations are conducted on full document pages without any limit.

To evaluate how well \ourmethod can generalize out-of-distribution beyond the document domain, we additionally evaluate its performance on \textbf{MME-RealWorld-Lite (MME-RW$_\ell$)} \citep{mme_rw} and \textbf{O3-Bench}~\citep{li2026insighto}. MME-RealWorld-Lite features a diverse set of high-resolution, real-world images, e.g., video monitoring, remote sensing, OCR in the wild, and self-driving corner cases~\cite{li2022coda}.
It has a mean image size of $2823 \times 1579$ pixels.
O3-Bench evaluates visual search, fine-grained perception, and multi-hop reasoning over high-resolution digital maps and composite charts. It tests how well an AI agent can truly ``think with images'' with interleaved attention to visual details.
It has a mean image size of $4602 \times 3967$ pixels.

\begin{table}[t]
\centering
\small
\setlength{\tabcolsep}{4pt}
\caption{\textbf{SFT trajectory-quality metrics.} Evidence metrics exclude unanswerable
rows and rows without evidence metadata. Page and region hit rates
also exclude zero-crop trajectories from their denominators.}
\label{tab:sft-quality}
\begin{tabular}{@{}p{0.70\columnwidth}r@{}}
\toprule
Metric & Value \\
\midrule
\multicolumn{2}{@{}l}{\textit{Scope}} \\
Rows with page evidence & 13{,}335 \\
Rows with box evidence & 6{,}925 \\
\addlinespace[2pt]
\multicolumn{2}{@{}l}{\textit{Coverage}} \\
Evidence-page hit rate & 95.33\% \\
Evidence-region hit rate, coverage $\geq 0.5$ & 85.00\% \\
Mean max evidence coverage & 85.24\% \\
\addlinespace[2pt]
\multicolumn{2}{@{}l}{\textit{Localization precision}} \\
Mean max crop/evidence IoU & 52.37\% \\
Crop region-hit rate & 65.34\% \\
Crops per evidence-region-hit crop & 1.53 \\
Crop area fraction & 14.50\% \\
\addlinespace[2pt]
\multicolumn{2}{@{}l}{\textit{Efficiency and redundancy}} \\
Same-source overlap rate, IoU $\geq 0.8$ & 3.16\% \\
Stuck rate & 0.99\% \\
Stop exactly at first region hit & 71.14\% \\
\bottomrule
\end{tabular}
\end{table}

\begin{figure*}[t]
    \centering     \includegraphics[width=\linewidth]{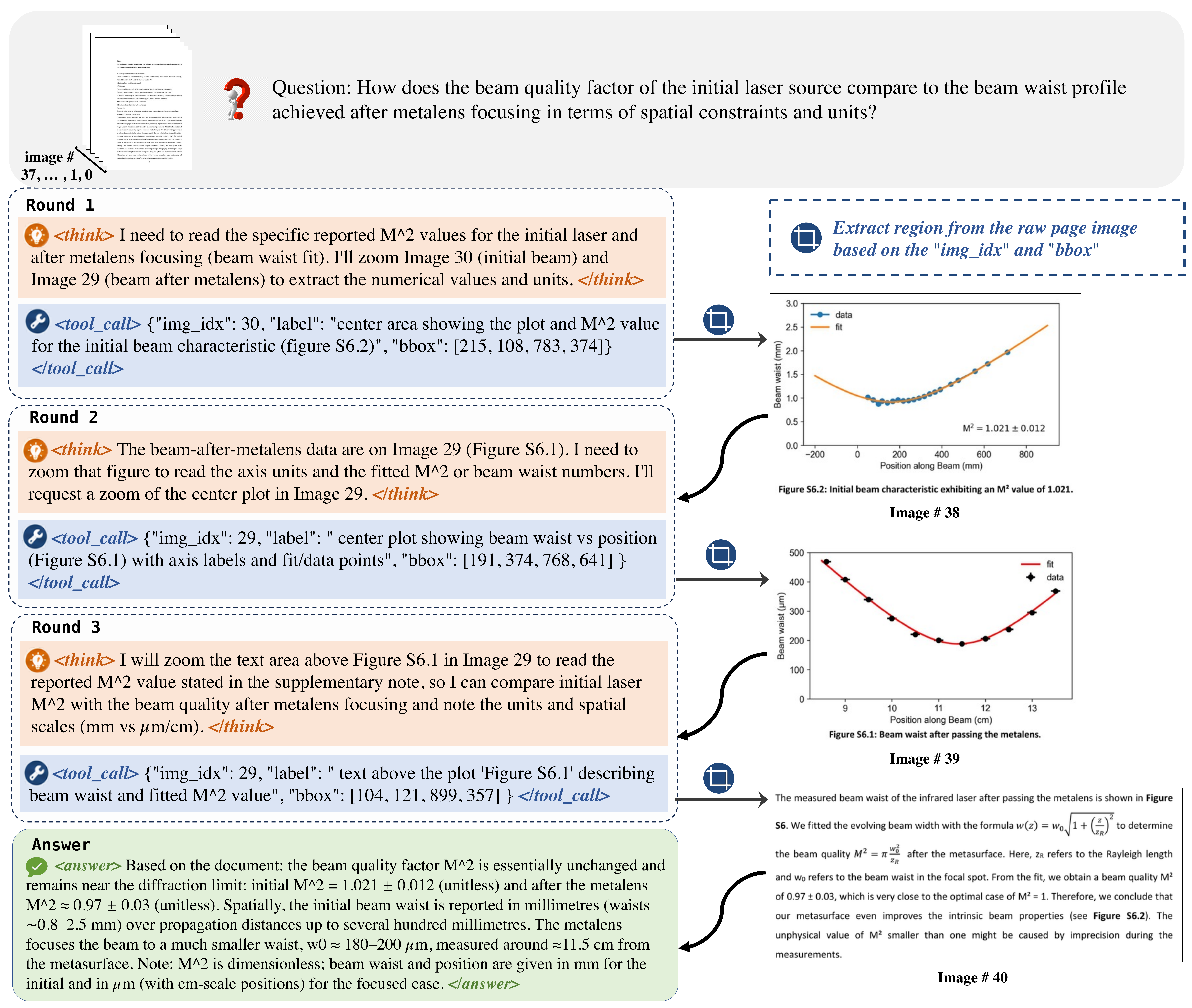}  
    
\caption{An example of our SFT data.}
    \label{fig:traing_data}
\end{figure*}

\subsection{Evaluation Metrics}
\label{app:longdoc_metric_definitions}

Unless otherwise stated, we report macro-averaged values of the following metrics over the stated evaluation benchmarks.

\begin{itemize}
    \item \textbf{Error rate} is defined as $1-\mathrm{accuracy}$, where \textbf{accuracy} is computed from the judged correctness of the model's final answer.
    
    \item \textbf{Hallucination rate} is computed only on unanswerable questions. It is the fraction of such questions for which the model provides a substantive answer instead of abstaining.
    
    \item \textbf{Sequence length} is the average total prompt-plus-response length per example, measured in tokens. This includes all (text/image) tokens including tool response tokens.
    
    \item \textbf{Latency} is the average end-to-end model inference time per example. For tool-using models, this includes all model calls and tool execution made during the multi-turn trajectory.
\end{itemize}

\subsection{Judge Calibration}
\label{app:judge-calibration}

Our reward model relies on an LLM judge to compare the model's final answer against the ground-truth answer. We first used \textbf{a legacy two-stage judge}. This judge asks the LLM to extract a concise final answer from the model trajectory and then asks a second LLM call to
verify whether the extracted answer matches the ground truth. This design is conservative and produced very few false positives, but it also produced \textit{many false negatives}. Manual inspection showed that the main issue was answer extraction: for long-form or multi-target
answers, the extractor sometimes collapsed a correct response into an overly generic phrase such as ``HIT applications listed'' or ``text describing Control Panel'', which then caused the verifier to reject an otherwise correct answer.

To reduce these false negatives, we next experimented with \textbf{a single-call judge}. Instead of separating extraction and verification, the single-call judge asks the LLM to directly decide whether the model answer should be considered correct. This recovered some correct
long-form answers, but it introduced \textit{a more serious failure mode for RL: false positives}. In particular, the single-call judge was more likely to credit incomplete answers or refusal-style answers on answerable questions. During RL, such false positives can become
exploitable reward signals, so this failure mode is more dangerous than sparse false negatives.

Motivated by these observations, we constructed \textbf{a 150-example judge calibration set with human labels}. The set is intentionally stress-tested rather than distribution-matched. It contains a variety of cases including legacy-correct cases, long-answer cases, multi-target list cases, unanswerable
questions, and \textit{25 answerable-refusal cases targeting the observed reward-hacking pattern}. To reduce the risk that the judge calibration overfits to one model's response style, the calibration set includes trajectories from \textit{multiple sources}, including Qwen3-VL-8B, earlier RL models, API models, and both training and evaluation examples.

We then designed a minimally modified legacy prompt, denoted \textbf{\texttt{legacy-v2}}. This judge keeps the two-stage extraction-and-verification structure of the legacy judge, but changes the prompts to make extraction more faithful to the final answer, avoid
inferring unsupported multiple-choice options, and reject answerable questions when the final response merely refuses to answer or states that the evidence is insufficient.

\begin{table}[h]
\centering
\small
\caption{Judge calibration on a 150-example manually labeled stress set using GPT-5-nano. Metrics are percentages except FP/FN counts.}
\label{tab:judge-calibration}
\begin{tabular}{lcccccc}
\toprule
Judge & Acc. & Prec. & Recall & F1 & FP & FN \\
\midrule
legacy & 84.7 & 98.1 & 70.3 & 81.9 & 1 & 22 \\
single-call & 87.3 & 86.7 & 87.8 & 87.2 & 10 & 9 \\
legacy-v2 & 94.7 & 97.1 & 91.9 & 94.4 & 2 & 6 \\
\bottomrule
\end{tabular}
\end{table}

As shown in Table~\ref{tab:judge-calibration}, the original legacy judge is precise but overly strict, while the single-call judge improves recall at the cost of many more false positives. The selected \texttt{legacy-v2} provides a better compromise: it
substantially reduces the legacy judge's false negatives while keeping false positives low. The remaining disagreements are mostly borderline cases, such as abbreviated table titles, terse unanswerable responses, or partially specified hierarchy titles.
\subsection{Hyperparameter Settings}
\label{app:hparam}
The key hyperparameters of the SFT and the RL of \ourmethod-8B are listed in Table~\ref{tab:hparam-sft} and Table~\ref{tab:hparam-rl}.

\begin{table}[h!]
    \centering
    \small
    \caption{Key hyperparameters for \ourmethod-8B SFT.}
    \label{tab:hparam-sft}
    \begin{tabular}{ll}
    \toprule
    \textbf{Hyperparameter} & \textbf{Value} \\
    \midrule
    Initialization & Qwen3-VL-8B-Instruct \\
    \# of training examples & 17,913 \\
    Training steps & 1118, 2 epochs \\
    Fine-tuning type & full-parameter fine-tuning \\
    Global batch size & 32 \\
    Max sequence length & 65,536 tokens \\
    Sequence parallel size & 4 \\
    Learning rate & $5\times10^{-6}$ \\
    LR schedule & cosine decay \\
    Warmup ratio & 0.05 \\
    Minimum learning rate & $5\times10^{-7}$ \\
    Optimizer & AdamW \\
    Weight decay & 0.01 \\
    Gradient clipping & 1.0 \\
    Vision encoder & frozen \\
    \bottomrule
    \end{tabular}
  \end{table}

\begin{table}[h!]
    \centering
    \small
    \caption{Key hyperparameters for \ourmethod-8B RL.}
    \label{tab:hparam-rl}
    \begin{tabular}{ll}
    \toprule
    \textbf{Hyperparameter} & \textbf{Value} \\
    \midrule
    Initialization & \ourmethod-8B (SFT) \\
    \# of training examples & 19,236 \\
    RL algorithm & GRPO \\
    Training steps & 800 \\
    Global batch size & 24 prompts \\
    Rollouts per prompt & 8 \\
    Effective rollout batch size & 192 responses \\
    Learning rate & $1\times10^{-6}$ \\
    Optimizer & AdamW \\
    KL regularization & low-var. KL, coeff{.} $0.01$ \\
    Max prompt length & 24,576 tokens \\
    Max response length & 8,192 tokens \\
    vLLM max model length & 32,768 tokens \\
    Vision encoder & frozen \\
    Sequence parallel size & 4 \\
    Rollout temperature & $0.7$ \\
    Rollout top-$p$ & $0.8$ \\
    Rollout top-$k$ & $20$ \\
    Rollout presence penalty & 1.5 \\
    Tool-use limit & 10 times \\
    Full-page image max area & $3500^2$ pixels \\
    Crop max area & $1280^2$ pixels \\
    Reward model & GPT-5-nano judge \\
    Reward weights & accuracy 1.0 \\
    Data sampler & weighted random refill \\
    \bottomrule
    \end{tabular}
\end{table}

\subsection{Inference Configuration for Evaluation}
Table~\ref{tab:hparam-inference} provides the inference configuration for \ourmethod-8B and open models.
For InternVL3-8B and GLM-4.6V-Flash, the max model length is set to their default context length.
For closed proprietary models, the inference is done via API.
The configuration follows the non-vLLM settings in Table~\ref{tab:hparam-inference}.

For long-document VQA benchmarks, i.e., MMLongBench-Doc and LongDocURL, we use a low-concurrency setting (4 workers times at most 1 concurrent job per worker) to reduce memory pressure.
For other benchmarks, we use a high-concurrency setting (8 workers times at most 4 concurrent jobs per worker).

\begin{table}[h!]
  \centering
  \small
  \caption{Key inference configuration for evaluation.}
  \label{tab:hparam-inference}
  \begin{tabular}{ll}
  \toprule
  \textbf{Configuration} & \textbf{Value} \\
  \midrule
  Inference backend & vLLM \\
  vLLM replicas & 4 \\
  GPUs per replica & 1 \\
  Number of agent workers & 8 (high) or 4 (low) \\
  Worker concurrency & 4 (high) or 1 (low) \\
  vLLM max model length & 262,144 tokens \\
  Max generated tokens & 16,384 tokens \\
  vLLM max batched tokens & 32,768 tokens \\
  vLLM max sequences & 64 \\
  GPU memory utilization & 0.8 \\
  Prefix caching & enabled \\
  Chunked prefill & enabled \\
  Sampling temperature & $0.7$ \\
  Sampling top-$p$ & $0.8$ \\
  Sampling top-$k$ & $20$ \\
  Presence penalty & 1.5 \\
  Repetition penalty & 1.0 \\
  Full-page image max area & $3500^2$ pixels \\
  Crop image max area & $1280^2$ pixels \\
  Region zoom factor & 2.0 \\
  Tool parser & Hermes-style tool parser \\
  Max parallel tool calls & 1 \\
  Tool-use limit & 10 times \\
  Context overflow handling & halve image area up to 4$\times$ \\
  \bottomrule
  \end{tabular}
\end{table}


\begin{table*}[t]
  \centering
  \caption{\textbf{Comparison with Qwen3-VL-8B under no page limit.} The models are compared under four initial resolution settings: \textit{low} ($r = 0.25$, DPI 50), \textit{medium-low} ($r = 0.35$, DPI 70), and \textit{medium} ($r = 0.5$, DPI 100), and \textit{high} ($r = 0.7$, DPI 140). $^\dagger$Page-limited results are included for reference.}
  \label{tab:uncapped_longdoc_results}
  \fontsize{8.5}{10}\selectfont
  \setlength\arrayrulewidth{0.6pt}
  \resizebox{\linewidth}{!}{%
  \begin{tabular}{l cccc cccc cccc}
      \toprule
        & \multicolumn{4}{c}{\textbf{MMLongBench-Doc}}
        & \multicolumn{4}{c}{\textbf{LongDocURL}}
        & \multicolumn{4}{c}{\textbf{Average}} \\
      \cmidrule(lr){2-5} \cmidrule(lr){6-9} \cmidrule(lr){10-13}
      \textbf{Model}
      & 0.25 & 0.35 & 0.5 & 0.7
      & 0.25 & 0.35 & 0.5 & 0.7
      & 0.25 & 0.35 & 0.5 & 0.7 \\
      \midrule
      Qwen3-VL-8B$^\dagger$    & 33.7 & 45.3 & 51.4 & 52.0 & 50.5 & 63.7 & 68.4 & 70.5 & 42.1 & 54.5 & 59.9 & 61.2 \\
      Qwen3-VL-8B   & 33.0 & 43.0 & 49.8 & 51.1 & 47.0 & 57.2 & 62.6 & 63.9 & 40.0 & 50.1 & 56.2 & 57.5 \\

      \midrule
      Qwen3-VL-8B (w/ zoom)$^\dagger$    & 33.2 & 40.4 & 48.8 & 53.1 & 47.1 & 58.2 & 63.8 & 68.8 & 40.2 & 49.3 & 56.3 & 60.9 \\
      Qwen3-VL-8B (w/ zoom)  & 33.6 & 40.5 & 47.8 & 49.9 & 36.7 & 46.6 & 57.7 & 62.2 & 35.2 & 43.5 & 52.7 & 56.0 \\

      \midrule
      \ourmethod-8B (SFT)$^\dagger$      & 36.5 & 43.0 & 48.0 & 47.3 & 57.0 & 61.1 & 63.7 & 66.0 & 46.7 & 52.0 & 55.9 & 56.6 \\
      \ourmethod-8B (SFT)    & 37.9 & 43.7 & 45.9 & 46.7 & 47.2 & 54.4 & 57.2 & 59.1 & 42.5 & 49.1 & 51.6 & 52.9 \\

      \midrule
      \ourmethod-8B (SFT+RL)$^\dagger$   & 50.8 & 55.6 & 58.6 & 57.9 & 63.3 & 68.4 & 70.5 & 71.9 & 57.0 & 62.0 & 64.5 & 64.9 \\
      \ourmethod-8B (SFT+RL) & 50.0 & 55.1 & 57.8 & 58.5 & 57.3 & 63.2 & 65.6 & 67.0 & 53.7 & 59.1 & 61.7 & 62.7 \\
      \bottomrule
  \end{tabular}%
  }
\end{table*}

\section{Additional Results and Discussion}

\subsection{Figure~\ref{fig:longdoc_efficiency} Details}
\label{app:longdoc_efficiency_figure}

Figure~\ref{fig:longdoc_efficiency} compares Qwen3-VL-8B-Instruct~\cite{Qwen3-VL} and InSight-doc-8B under matched and reduced input resolutions.
The reported metrics are defined in Appendix~\ref{app:longdoc_metric_definitions}.
All the results except hallucination rates are macro averages over MMLongBench-Doc~\cite{ma2024mmlongbench} and LongDocURL~\cite{deng2025longdocurl}.
The hallucination rates are macro averages over MMLongBench-Doc and DUDE~\cite{van2023document}.

\subsection{Long-Document VQA without Page Limit}
\label{app:uncapped_longdoc_eval}

Table~\ref{tab:uncapped_longdoc_results} compares each model under the standard capped long-document setting and the uncapped setting, where all document pages are retained. The uncapped rows are consistently lower than their capped counterparts in macro average,
showing that the page cap hides part of the difficulty of long-document understanding. The drop is driven mostly by LongDocURL: for example, at $r=0.7$, LongDocURL drops from 70.5 to 63.9 for Qwen3-VL without tools, from 68.8 to 62.2 with zoom, from 66.0 to 59.1 for
SFT, and from 71.9 to 67.0 for SFT+RL. In contrast, MMLongBench-Doc changes much less, suggesting that the additional uncapped pages in LongDocURL introduce more distractors or require broader evidence aggregation.

The relative ordering is stable across the capped and uncapped settings. The RL model remains the strongest model at every resize factor, and its capped-to-uncapped macro drop is smaller than the other tool-using models, especially at high resolution: at $r=0.7$,
SFT+RL drops by 2.2 points, compared with 3.7 points for SFT and 4.9 points for the base model with zoom. This suggests that the learned zoom policy is not only more accurate under the standard capped evaluation, but also more robust when the document is fully
exposed.

\subsection{Detailed Comparison with Related Methods}
\label{app:comp-w-related-methods}

\begin{table*}[t]
\centering
\small
\setlength{\tabcolsep}{4pt}
\renewcommand{\arraystretch}{1.12}
\caption{
\textbf{Expanded comparison with retrieval-based, coarse-to-fine, iterative, and structured document reasoning methods.}
Scores are official cross-paper numbers and are used for positioning rather than controlled head-to-head comparison.
\textbf{R-f} (retriever-free) means no external page/document retriever is used.
\textbf{C2F} (coarse-to-fine) means the method starts from coarse document view and then acquires finer evidence.
\textbf{Itr} (iterative) means explicit multi-turn evidence acquisition.
\textbf{Rgn} (region) means region-level evidence acquisition rather than page-level.
\textbf{MPDoc.}, \textbf{MMLD.}, and \textbf{LDoc.} refer to MP-DocVQA, MMLongBench-Doc, and LongDocURL, respectively.
$^\dagger$Prior work commonly reports ANLS on DUDE and MP-DocVQA, while our InSight-doc results are evaluated with the same LLM-as-judge protocol used for our main experiments across benchmarks.
}
\label{tab:appendix_related_full}
\begin{tabular}{llccccccccc}
\toprule
\textbf{Method} &
\textbf{Backbone} &
\makecell{\textbf{Param.}} &
\makecell{\textbf{R-f}} &
\makecell{\textbf{C2F}} &
\textbf{Itr} &
\makecell{\textbf{Rgn}} &
\textbf{DUDE} &
\textbf{MPDoc.} &
\textbf{MMLD.} &
\textbf{LDoc.} \\
\midrule
\multicolumn{11}{l}{\textcolor{gray}{\textit{Retrieval-based methods}}} \\

GPT-4o + ColPali &
GPT-4o &
-- &
\xmark & \xmark & \xmark & \xmark &
-- & -- & 30.8 & -- \\

CREAM &
Pix2Struct + LLaMA2 &
7B &
\xmark & \cmark & \xmark & \xmark &
52.5 & 74.3 & -- & -- \\

M3DocRAG &
Qwen2-VL &
7B &
\xmark & \xmark & \xmark & \xmark &
39.5 & 84.4 & 21.0 & 35.1 \\

VisRAG &
MiniCPM-V 2.6 &
8B &
\xmark & \xmark & \xmark & \xmark &
43.1 & -- & 18.8 & 41.9 \\

SV-RAG &
InternVL2 &
4B &
\xmark & \xmark & \xmark & \xmark &
45.0 & 71.0 & 23.0 & -- \\

VDocRAG &
Phi3-Vision &
4B &
\xmark & \xmark & \xmark & \xmark &
44.0 & 62.6 & 18.4 & 39.8 \\

MoLoRAG &
Qwen2.5-VL &
7B &
\xmark & \xmark & \xmark & \xmark &
-- & -- & 41.0 & 51.9 \\

URaG &
Qwen2.5-VL &
7B &
\xmark & \cmark & \xmark & \xmark &
57.6 & 88.2 & 33.8 & 52.2 \\

\midrule
\multicolumn{11}{l}{\textcolor{gray}{\textit{Iterative, coarse-to-fine, and structured reasoning methods}}} \\

Doc-React &
GPT-4o &
-- &
\xmark & \xmark & \cmark & \xmark &
-- & -- & 38.3 & -- \\

VRAG-RL &
Qwen2.5-VL &
7B &
\xmark & \cmark & \cmark & \cmark &
-- & -- & 26.6 & 44.9 \\

CogDoc &
Qwen2.5-VL &
7B &
\cmark & \cmark & \xmark & \xmark &
46.2 & 75.0 & 33.0 & -- \\

DocR1 &
Qwen2.5-VL &
7B &
\cmark & \cmark & \xmark & \xmark &
54.4 & 87.5 & -- & -- \\

DocSeeker &
Qwen2.5-VL &
7B &
\cmark & \xmark & \xmark & \xmark &
57.4 & 86.2 & 40.1 & 51.7 \\

Doc-V$^{\star}$ &
Qwen2.5-VL &
7B &
\xmark & \cmark & \cmark & \xmark &
64.5 & 86.2 & 42.1 & 56.3 \\

MM-Doc-R1 &
Qwen3 + Qwen2.5-VL &
8B &
\xmark & \cmark & \cmark & \xmark &
-- & -- & 49.7 & -- \\

\midrule
\textbf{InSight-doc ($r=0.25$)} &
\textbf{Qwen3-VL} &
\textbf{8B} &
\cmark & \cmark & \cmark & \cmark &
\hphantom{$^\dagger$}\textbf{70.1}$^\dagger$ & \hphantom{$^\dagger$}\textbf{83.4}$^\dagger$ & \textbf{50.0} & \textbf{57.3} \\

\textbf{InSight-doc ($r=0.35$)} &
\textbf{Qwen3-VL} &
\textbf{8B} &
\cmark & \cmark & \cmark & \cmark &
\hphantom{$^\dagger$}\textbf{72.1}$^\dagger$ & \hphantom{$^\dagger$}\textbf{86.3}$^\dagger$ & \textbf{55.1} & \textbf{63.2} \\

\textbf{InSight-doc ($r=0.5$)} &
\textbf{Qwen3-VL} &
\textbf{8B} &
\cmark & \cmark & \cmark & \cmark &
\hphantom{$^\dagger$}\textbf{73.8}$^\dagger$ & \hphantom{$^\dagger$}\textbf{87.6}$^\dagger$ & \textbf{57.8} & \textbf{65.6} \\

\textbf{InSight-doc ($r=0.7$)} &
\textbf{Qwen3-VL} &
\textbf{8B} &
\cmark & \cmark & \cmark & \cmark &
\hphantom{$^\dagger$}\textbf{73.8}$^\dagger$ & \hphantom{$^\dagger$}\textbf{88.2}$^\dagger$ & \textbf{58.5} & \textbf{67.0} \\

\bottomrule
\end{tabular}
\end{table*}

Table~\ref{tab:appendix_related_full} provides an expanded comparison with recent visual-retrieval, RAG, and coarse-to-fine document reasoning methods.
Although these methods are evaluated on overlapping long-document benchmarks, their reported numbers are not directly comparable because they differ in several important settings.
In particular, DUDE and MP-DocVQA numbers are not strictly metric-aligned across papers, since many prior methods report ANLS whereas we use LLM-as-judge accuracy for consistency across extractive, reasoning, and unanswerable questions.

Visual RAG methods, including ColPali (w/ GPT-4o)~\citep{faysse2024colpali},
CREAM~\citep{zhang2024cream},
M3DocRAG~\citep{cho2024m3docrag}, VisRAG~\citep{yu2024visrag}, SV-RAG~\citep{chen2024sv}, VDocRAG~\citep{tanaka2025vdocrag}, MoLoRAG~\citep{wu2025molorag}, and URaG~\citep{shi2026urag} , primarily reduce context by selecting relevant pages, document images, or text chunks before answer generation.
Their central question is which evidence units should be retrieved.
This differs from our setting, where the full document is already available at low resolution and the model decides which sub-page regions require higher visual resolution during reasoning.

Iterative evidence-acquisition methods such as Doc-React~\citep{wu2025doc}, VRAG-RL~\citep{wang2026vrag}, and MM-Doc-R1~\citep{lin2026mm} are closer to our agentic setting.
Doc-React iteratively refines retrieval queries, while MM-Doc-R1 uses a hybrid agent setup with a Qwen3~\citep{yang2025qwen3} reasoning backbone and a Qwen2.5-VL~\citep{bai2025qwen2.5vl} reader.
VRAG-RL is the closest among these methods in that it can perform multi-step visual evidence acquisition with search, crop, and scale actions.
However, it is still formulated as a visual RAG/search-agent pipeline, whereas InSight-doc is retriever-free: it starts from a low-resolution view of the full target document and dynamically zooms into high-resolution sub-page regions during reasoning.

CogDoc~\citep{xu2025cogdoc}, DocR1~\citep{xiong2025docr1}, DocSeeker~\citep{yan2026docseeker}, and Doc-V$^{\star}$~\citep{zheng2026docvstar}, are the closest coarse-to-fine, structured, or agent-style document reasoning methods.
CogDoc and DocR1 encourage evidence-page identification before answering, DocSeeker produces structured Analysis--Localization--Reasoning traces with page-level evidence, and Doc-V$^{\star}$ fetches high-resolution pages ($1024\times768$ pixels/page; equivalent to $r\approx0.13$) after a coarse overview ($256\times256$ pixels/page; equivalent to $r\approx0.46$).
Doc-V$^{\star}$ mostly relies on an external retriever to fetch the pages.

Finally, these methods differ in backbones, training data, preprocessing, input resolution, page limits, retrieval method, tool budgets, and evaluation code.
Therefore, we use official cross-paper numbers for positioning only.
A fully controlled comparison would require rerunning each method under the same backbone, document preprocessing, resolution, page/tool budget, and evaluation script, which is not possible for the closest unreleased methods such as CogDoc and Doc-V$^{\star}$.

\subsection{Further Comparison with Doc-V$^{\star}$}

Doc-V$^{\star}$~\citep{zheng2026docvstar} is the closest related method to \ourmethod in spirit, since it also starts from a low-resolution document overview and interleaves evidence acquisition with reasoning.
However, a fully controlled empirical comparison is not currently feasible because \emph{Doc-V$^{\star}$ has not released its training dataset, model checkpoints, or training/evaluation code}.
More importantly, Doc-V$^{\star}$ studies a different setting: its agent has access to an external retriever, ColQwen2.5~\citep{faysse2024colpali}, whereas InSight-doc is an end-to-end agent that decides where to zoom without any external page retriever.

This difference is substantial in practice.
As reported in Table~5 of the Doc-V$^{\star}$ paper, 94.0--99.8\% of its trajectories call the external retriever, while only 3.4--14.7\% involve page-level retrieval by the agent itself from the document overview.
Without the retriever, its accuracy drops from 39.8\% to 34.9\% on MMLongBench-Doc (see Table 4 of the Doc-V$^{\star}$ paper).
In contrast, InSight-doc makes no external retriever calls, and 81.6--99.3\% of its trajectories involve region-level evidence acquisition by the agent itself.
\textit{Thus, the main distinction is not only region-level versus page-level evidence, but also end-to-end visual evidence acquisition versus retrieval-assisted interaction.}

For a more controlled comparison, we conduct a controlled ColPali-style experiment to further approximate Doc-V$^{\star}$'s retrieval-assisted page-level interaction.
We use the same external retriever as Doc-V$^{\star}$, namely ColQwen2.5, a ColPali variant based on Qwen2.5-VL-3B-Instruct.
We retrieve large page sets with $K \in \{8,16,32\}$ to mimic the evidence available after multiple retrieval turns, and feed the retrieved pages into Qwen3-VL-8B-Instruct, the same base model used by InSight-doc-8B.
Although this is still not a fully apples-to-apples comparison with Doc-V$^{\star}$, it controls the answer-generation backbone and directly tests retrieval-assisted page-level evidence selection.

The results are shown in Table~\ref{tab:static-retrieval}. At all comparable budget levels, InSight-doc achieves higher accuracy with fewer input tokens.
These results suggest that end-to-end region-level adaptive zooming is at least competitive with, and likely more token-efficient than, retrieval-assisted page-level interaction used by concurrent RAG-based methods such as Doc-V$^{\star}$.

\subsection{Additional Performance Analysis}

On the longest-document subset, MMLongBench-Doc increases from 47.7 to 71.0 pages, and LongDocURL from 89.0 to 106.5 pages. Increasing page
count reduces accuracy and increases sequence length for both Qwen3-VL-8B without tools and \ourmethod-8B. Nevertheless, \ourmethod-8B remains consistently stronger as shown in Figure~\ref{fig:expriment_longest}. Averaged over MMLongBench-Doc and LongDocURL, its gains over Qwen3-VL-8B on the longest-document subset are +11.6, +6.4, and +6.4 points at $r=0.25,0.35,0.5$, respectively.

Compared with the results on the full evaluation set, the accuracy-efficiency tradeoff becomes slightly more favorable for \ourmethod-8B on longer documents. On the full set, \ourmethod-8B at $r=0.35$ exceeds Qwen3-VL-8B at $r=0.7$ in macro accuracy, 59.1\% vs. 57.5\%, while reducing latency from 29.6s to 9.3s
and sequence length from 111.7k to 33.4k tokens. On the longest-document subset, this comparison improves to 56.2\% vs. 53.2\%, with latency reduced from 39.3s to 11.2s and sequence length from 136.8k to 42.4k tokens. Thus, as page count increases, targeted inspection
preserves or improves accuracy while avoiding most of the cost of uniformly increasing full-document resolution.

\begin{figure}[t]
    \vspace{1mm}
    \centering
    \includegraphics[width=\linewidth]{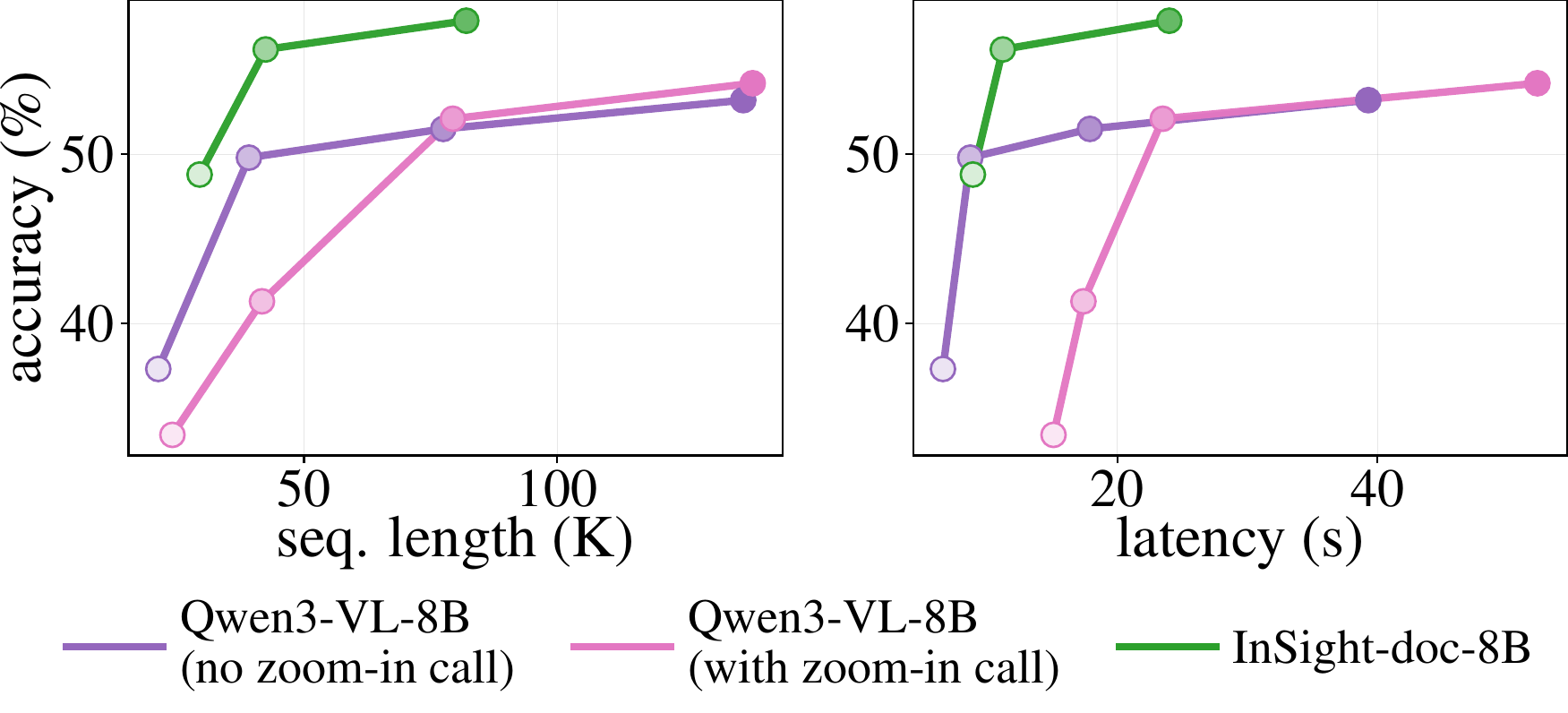}   
    \vspace{-5mm}
    \caption{
\textbf{Accuracy vs.\ efficiency on the longest-document examples} from MMLongBench-Doc and LongDocURL (200 examples each).
InSight-doc-8B achieves higher accuracy at substantially lower time and token cost than both baselines.
}
    \label{fig:expriment_longest}
\end{figure}

\section{\ourmethod Output Examples}
\label{app:example_appendix}

More inference examples of \ourmethod are provided in Figures \ref{fig:teaser5},  \ref{fig:teaser3}, \ref{fig:teaser4-1}, \ref{fig:teaser4-2}, and \ref{fig:teaser2}.

\section{Information About Use of AI Assistants}
We used AI assistants (e.g., GPT and Gemini) for literature review, manuscript polishing, generating codes, and launching experiments.

\begin{table*}[t]
\centering
\caption{
Controlled proxy comparison with retrieval-assisted page-level evidence selection on the 200 longest documents from each of MMLongBench-Doc and LongDocURL.
The ColQwen2.5 baseline retrieves the top-$K$ pages and feeds them to Qwen3-VL-8B-Instruct, approximating the external-retriever setting used by Doc-V$^{\star}$ while controlling the answer-generation backbone.  
}
\label{tab:static-retrieval}
\fontsize{8.5}{10}\selectfont
\begin{tabular}{llcc}
\toprule
\textbf{Method} & \textbf{Setting} & \textbf{Acc. (\%)} & \textbf{Avg. Tokens} \\
\midrule
Qwen3-VL-8B + ColQwen2.5 & Top-8 pages  & 43.4 & $\sim$28K \\
InSight-doc-8B           & 50 DPI       & 48.8 & $\sim$22K \\
\midrule
Qwen3-VL-8B + ColQwen2.5 & Top-16 pages & 49.1 & $\sim$57K \\
InSight-doc-8B           & 70 DPI       & 56.2 & $\sim$41K \\
\midrule
Qwen3-VL-8B + ColQwen2.5 & Top-32 pages & 53.5 & $\sim$104K \\
InSight-doc-8B           & 100 DPI      & 57.9 & $\sim$81K \\
\bottomrule
\end{tabular}
\end{table*}

\begin{figure*}[t]
\centering
\includegraphics[width=\linewidth]{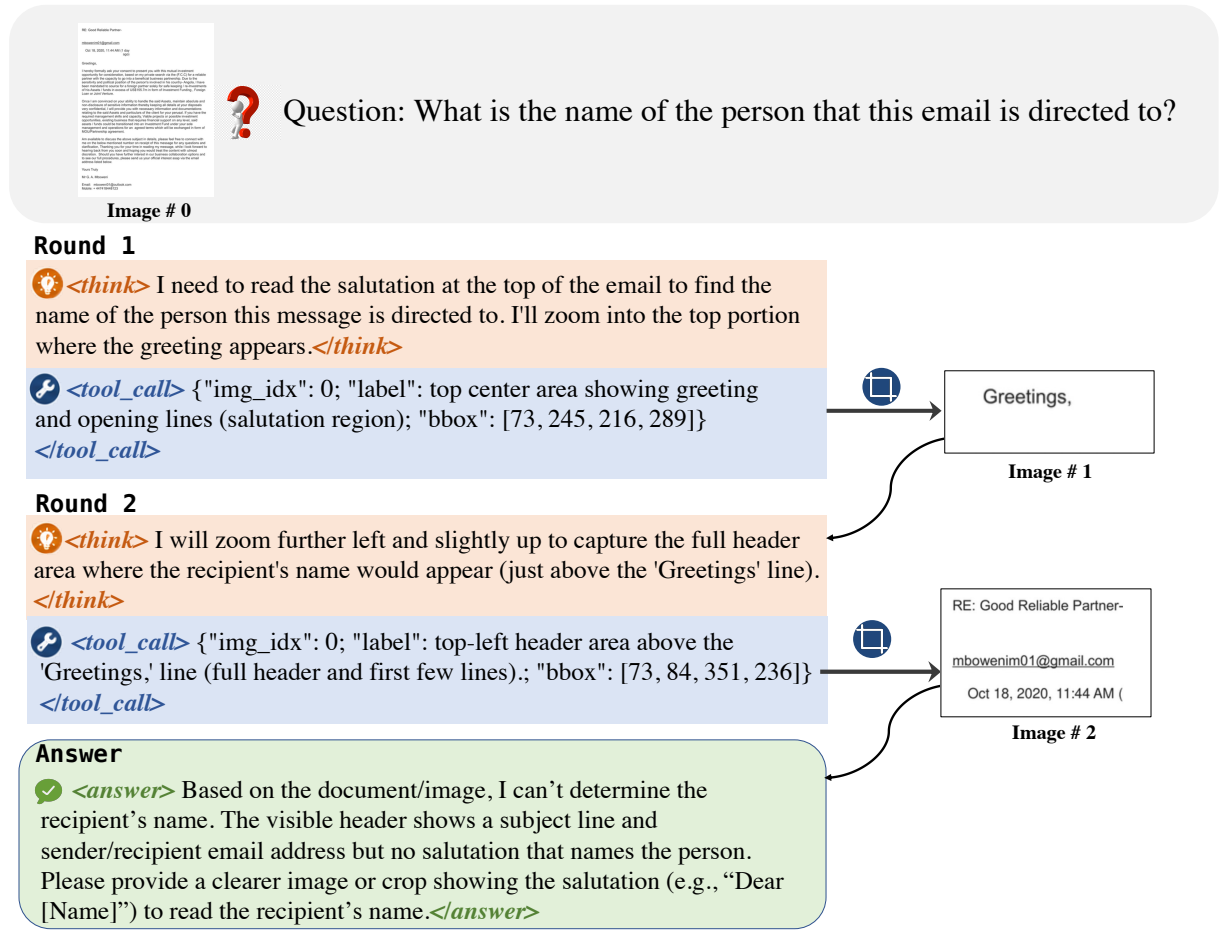}  
\caption{Example of \ourmethod{} on an unanswerable question, i.e., one whose answer cannot be supported by any evidence in the document.}
\label{fig:teaser5}
\end{figure*}

\begin{figure*}[t]
    \centering
    \includegraphics[width=\linewidth]{figure/appendix/zoomin1.pdf}

    \vspace{1cm} 
    
    \includegraphics[width=\linewidth]{figure/appendix/zoomin2.pdf}
\caption{Examples illustrating \ourmethod's sequential zoom-in behavior, where it progressively refines target regions until sufficient evidence is collected. The bottom example, taken from MME-RealWorld-Lite, shows that \ourmethod{} also generalizes to natural-image visual search despite being trained solely on document data.}
    \label{fig:teaser3}
\end{figure*}

\begin{figure*}[t]
    \centering     \includegraphics[width=\linewidth]{figure/appendix/humanlike.pdf}     
\caption{An example demonstrating that \ourmethod can think and see in a manner similar to humans:adjusting the focus when the initially identified region is imprecise. }
    \label{fig:teaser4-1}
\end{figure*}

\begin{figure*}[t]
    \centering
    \includegraphics[width=\linewidth]{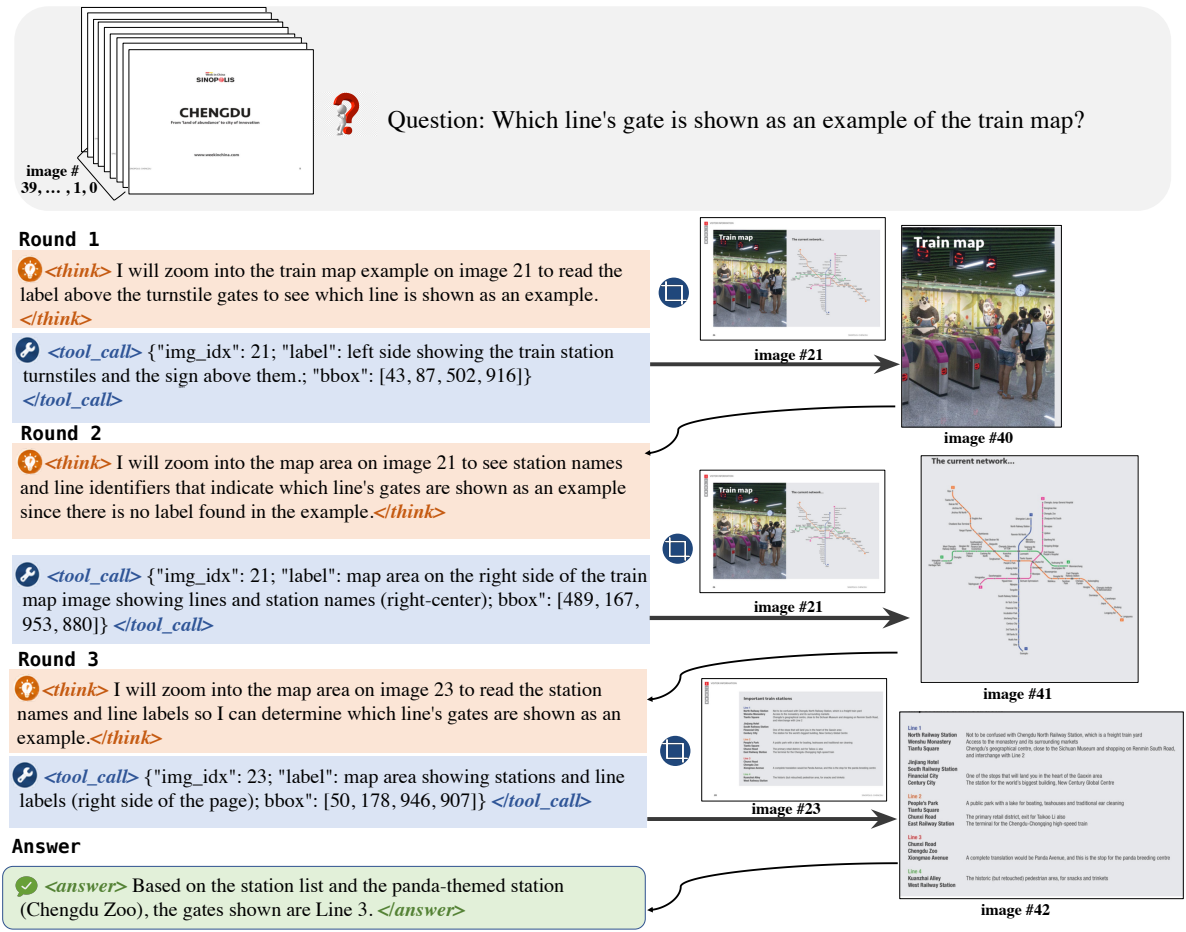}    
\caption{An example demonstrating that \ourmethod can think and see in a manner similar to humans: actively exploring potential regions that may contain the answer. }
    \label{fig:teaser4-2}
\end{figure*}

\begin{figure*}[t]
    \centering   
    \includegraphics[width=1.\linewidth]{figure/appendix/crosspage2.pdf}  
    
     \vspace{1cm} 
     
    \includegraphics[width=\linewidth]{figure/teaser_0527.pdf} 
       \caption{Examples demonstrating that \ourmethod can gather evidence from different pages.}
    \label{fig:teaser2}
\end{figure*}

\end{document}